\documentclass[letterpaper]{article} % DO NOT CHANGE THIS
\usepackage{aaai2026}  % DO NOT CHANGE THIS
\usepackage{times}  % DO NOT CHANGE THIS
\usepackage{helvet}  % DO NOT CHANGE THIS
\usepackage{courier}  % DO NOT CHANGE THIS
\usepackage[hyphens]{url}  % DO NOT CHANGE THIS
\usepackage{graphicx} % DO NOT CHANGE THIS
\usepackage{natbib}  % DO NOT CHANGE THIS AND DO NOT ADD ANY OPTIONS TO IT
\usepackage{caption} % DO NOT CHANGE THIS AND DO NOT ADD ANY OPTIONS TO IT
\usepackage{algorithm}
\usepackage{algorithmic}

\usepackage{amsfonts}
\usepackage{subcaption}
\usepackage{amsmath}
\usepackage{amsthm}
\usepackage{booktabs}
\usepackage{multirow}

\newtheorem{lemma}{Lemma}
\newtheorem{theorem}{Theorem}

\newtheorem{definition}{Definition}

\newcommand{\Ttau}{{T}}

\makeatletter
\newcommand{\saferef}[2]{%
    \@ifundefined{r@#1}{#2}{\ref{#1}}%
}

\usepackage{newfloat}
\usepackage{listings}
\DeclareCaptionStyle{ruled}{labelfont=normalfont,labelsep=colon,strut=off} % DO NOT CHANGE THIS
\floatstyle{ruled}
\newfloat{listing}{tb}{lst}{}
\floatname{listing}{Listing}
\title{SafeMIL: Learning Offline Safe Imitation Policy from Non-Preferred Trajectories}
\author{
   Returaj Burnwal\textsuperscript{\rm 1},
   Nirav Pravinbhai Bhatt\textsuperscript{\rm 2},
   Balaraman Ravindran\textsuperscript{\rm 2}
}
\affiliations{
  \textsuperscript{\rm 1}Department of Computer Science and Engineering,\\
  \textsuperscript{\rm 2}Department of Data Science and AI,  Wadhwani School of Data Science and AI\\
  Indian Institute of Technology Madras, India\\
   cs21d406@smail.iitm.ac.in, \{niravbhatt, ravi\}@dsai.iitm.ac.in
}

\begin{document}

\maketitle

\begin{abstract}
In this work, we study the problem of offline safe imitation learning (IL). In many real-world settings, online interactions can be risky, and accurately specifying the reward and the safety cost information at each timestep can be difficult. However, it is often feasible to collect trajectories reflecting undesirable or risky behavior, implicitly conveying the behavior the agent should avoid. We refer to these trajectories as non-preferred trajectories. Unlike standard IL, which aims to mimic demonstrations, our agent must also learn to avoid risky behavior using non-preferred trajectories. In this paper, we propose a novel approach, SafeMIL, to learn a parameterized cost that predicts if the state-action pair is risky via \textit{Multiple Instance Learning}. The learned cost is then used to avoid non-preferred behaviors, resulting in a policy that prioritizes safety. We empirically demonstrate that our approach can learn a safer policy that satisfies cost constraints without degrading the reward performance, thereby outperforming several baselines.
\end{abstract}

\section{Introduction}
Reinforcement Learning (RL) has demonstrated significant success across various challenging tasks \cite{roboticsRL,humanDeepRL,masteringGo,autoDriveRL}. Nevertheless, its deployment in real-world applications is limited because most RL algorithms require many online interactions with the environment to learn a good policy. This requirement hinders its use in many real-world domains, especially where such interactions are risky, such as robotics or autonomous driving. Another challenge lies in selecting a suitable reward function for complex real-world tasks. RL agent learns a policy by maximizing the reward; however, crafting a reward function that encodes the desired behavior can be challenging in various practical applications. An ill-defined reward function can lead to unintended and potentially harmful behaviors \cite{ill_defined_reward_aaai_2014,ill_defined_reward_iclr_2020}. In contrast to RL, which relies on environmental rewards, Imitation Learning (IL) \cite{il_1_neurips_1996} has the distinct advantage of learning solely from expert demonstrations, eliminating the need to design a reward function. 

In most real-world domains, agents should also adhere to safety constraints alongside reward maximization. Safety in RL is usually modeled as a Constrained Markov Decision Process (CMDP) \cite{cmdp_1999,safety_gym_2019,safety_gymnasium_neurips_2023}, where a policy is considered safe if it satisfies all the safety constraints. Offline Safe RL \cite{copo_icml_2022,cdt_icml_2023} methods can learn safe policies from pre-collected datasets. However, these methods require access to constraint cost information at each timestep, which can be challenging to obtain. For example, assessing the toxicity in conversational agents or defining safety constraints in complex domains like surgical robotics can be difficult. In contrast, labeling a trajectory non-preferred is easier than defining constraint costs at each timestep. Furthermore, non-preferred trajectories that violate safety constraints are often naturally collected. For instance, vehicle accidents and toxic content reported by chatbot users can be used as non-preferred trajectories. 

% In this paper, we address the problem of offline safe imitation learning, where given limited non-preferred trajectories and a large number of unlabeled trajectories, which include both preferred (i.e., high return, safety constraint satisfying) and non-preferred (i.e., high return, safety constraint violating) trajectories. The agent aims to learn a safe policy that follows preferred behavior while avoiding non-preferred behavior. 
In this paper, we address the problem of offline safe imitation learning, where the agent learns a safe policy by avoiding non-preferred behaviors. Our approach relies on a small number of non-preferred trajectories, along with a large collection of unlabeled trajectories that contain a mix of preferred (i.e., high return, safety-constraint satisfying) and non-preferred (i.e., high return but safety-constraint violating) trajectories, where both per-timestep reward and cost information are unknown.
For example,  consider a real-world scenario of autonomous driving where we have access to limited non-preferred trajectories (e.g., red-light running, proximity to other vehicles) and a large unlabeled dataset of human-driven trajectories, which may include non-preferred trajectories. Given no further online data collection, the challenge is to learn a safe policy from these offline datasets. 
% We assume that both the preferred and non-preferred trajectories are high-return trajectories.

We propose a novel \textit{offline Safe imitation learning via Multiple Instance Learning}, SafeMIL, algorithm that can learn a safe policy using limited non-preferred trajectories and a large number of unlabeled trajectories where both reward and cost information are unavailable. Our key contributions are as follows: 1) We formulate the learning of a parameterized cost function that predicts if the state-action pair is risky as a Multiple Instance Learning (MIL) problem. Our work is the first to introduce the MIL formulation for offline safe IL setting. 2) Our method can learn the cost function even with a simple and intuitive score function, equation \ref{eq:bag_score}. The learned cost is then used to identify preferred behavior in the unlabeled dataset. This allows the agent to learn a policy via behavior cloning that imitates behaviors likely to satisfy CMDP constraints. 
% This learned cost function is subsequently used as weights in the weighted behavior cloning loss function, resulting in a safe policy that avoids non-preferred behaviors. 
3) We empirically demonstrate that our proposed algorithm learns significantly safer policies. It consistently outperforms or matches the state-of-the-art offline safe imitation learning methods in terms of safety, achieving final median performance that is $3.7\times$ better than the best baseline algorithm across all environments (See Appendix Table \saferef{table:performance_velocity}{3}, \saferef{table:performance_navigation}{4}).

% We propose a novel offline safe IL algorithm that can learn safe policy using limited non-preferred trajectories and a large number of unlabeled trajectories. Our approach, \textit{offline Safe imitation learning via Multiple Instance Learning} (SafeMIL), learns a parameterized state-action-based cost function such that the expected total safety cost of the non-preferred trajectory is higher than that of the unlabeled trajectory. Then, this learned cost is subsequently used as the weights in the weighted behavior cloning loss function, resulting in a safe policy that avoids non-preferred behaviors. We empirically demonstrate that our proposed algorithm learns safe policies that are superior or competitive with state-of-the-art offline safe imitation learning algorithms in terms of safety.

\section{Related Works}
Offline Imitation Learning \cite{bc,demodice_iclr_2022} primarily focuses on replicating actions of the demonstrations without explicitly considering safety. It implicitly assumes that demonstrations are safe, but if they contain non-preferred trajectories, simply imitating them can lead to learning policies that are risky. We address the problem of learning a safe imitation policy by seeking limited access to non-preferred trajectories and an abundance of unlabeled trajectories that contain both preferred and non-preferred trajectories. This specific scenario has received limited attention in the literature.

\paragraph{Learning Imitation Policy from Non-Preferred Trajectories}: SafeDICE \cite{safedice_neurips_2023} addresses this challenge by directly estimating the stationary distribution of the preferred policy. The estimated distribution is then used to learn a safe policy that mimics the preferred behavior while effectively avoiding the non-preferred behaviors.

\paragraph{Learning Imitation Policy from Suboptimal Trajectories}: T-REX \cite{trex_icml_2019}, B-Pref \cite{b_pref_neurips_2021}, PEBBLE \cite{pebble_icml_2021}, OPRL \cite{offline_pref_tmlr_2023}, learn a reward function from ranked trajectories. These methods focus on learning reward functions that assign a higher total reward to higher-ranked trajectories in the dataset. T-REX, B-Pref, PEBBLE then uses online-RL \cite{ppo,sac} algorithm, and OPRL uses offline-RL \cite{cql_neurips_2020,neorl_neurips_2022} algorithm to learn the policy. Since ranking trajectories can be challenging, D-REX \cite{drex_corl_2020} and SSRR \cite{ssrr_corl_2020} generate trajectories of varying quality by injecting noise into policies learned from suboptimal demonstrations. These approaches then use the resulting trajectories to learn a reward function. However, trajectory generation requires online interaction with the environment, which is not viable in an offline setting.
Another approach, DWBC \cite{dwbc_icml_2022}, uses positive-unlabeled (PU) learning \cite{pu_learning_1_2008,pu_learning_2_2021} to train a discriminator network, which is then incorporated as a weight in the BC loss function.

\section{Preliminaries}
\subsection{Constrained Markov Decision Process}
Consider a finite Constrained Markov Decision Process (CMDP) \cite{cmdp_1999} represented by the tuple $\mathcal{M} = (\mathcal{S},\mathcal{A},\mathcal{P}, r, \rho_0, \gamma, \mathcal{C})$, where $\mathcal{S}$ and $\mathcal{A}$ represent state and action spaces, respectively. $\mathcal{P}(s_{t+1} | s_t, a_t)$ defines the probability of transitioning to state $s_{t+1}$ after executing action $a_t$ in state $s_t$ at timestep $t$. $r: \mathcal{S}\times\mathcal{A} \rightarrow \mathbb{R}$ denotes the immediate reward, $\rho_0$ is the initial state distribution, and $\gamma \in (0,1)$ is the discount factor. $\mathcal{C} = \{(c_i,b_i)\}_{i=1}^m$ is a constraint set, where $c_i: \mathcal{S}\times\mathcal{A} \rightarrow \mathbb{R}_{\geq 0}$ is the $i$-th cost function and $b_i\in\mathbb{R}_{\geq 0}$ is the corresponding threshold. A policy $\pi: \mathcal{S} \rightarrow P(\mathcal{A})$ corresponds to a map from state to a probability distribution over actions. 

The set of feasible stationary policies for a CMDP is:
\begin{align} \label{eq:constraint_set}
    \Pi_\mathcal{C} := \left\{\pi\; \big| \; \forall i,\; \mathbb{E}_{\tau \sim \pi} \left[\sum_{t=0}^{\Ttau-1} \gamma^t c_i(s_t, a_t)\right] \leq b_i \right\}
\end{align}
where $\mathbb{E}_{\tau \sim \pi} \left[\sum_{t=0}^{\Ttau-1} \gamma^t c_i(s_t, a_t)\right] \leq b_i$ is the $i$-th constraint and $\tau = (s_0, a_0,s_1,a_1 \dots, s_{\Ttau-1}, a_{\Ttau -1}, s_\Ttau)$ is a $\Ttau$ length trajectory sampled under policy $\pi$. The reinforcement learning problem in CMDP is then to find a safe optimal policy $\pi^*$ that maximizes the expected discounted cumulative reward while satisfying all constraints:
\begin{align} \label{eq:CMDP}
    \pi^\star = arg\max_{\pi\in \Pi_\mathcal{C}}\; \mathbb{E}_{\tau \sim \pi} \left[\sum_{t=0}^{\Ttau-1} \gamma^t r(s_t, a_t)\right]   
\end{align}

Imitation learning does not rely on environmental rewards but rather on demonstrations. This work focuses on an offline setting, where the policy is learned solely from the pre-collected trajectories that do not contain reward and cost information explicitly, i.e, $\tau = (s_0, a_0,s_1,a_1 \dots, s_{\Ttau-1}, a_{\Ttau -1}, s_\Ttau)$. In this work, preferred and non-preferred trajectories are defined as follows:
\begin{definition}[Preferred Trajectory]
A preferred trajectory is defined as a trajectory that achieves high returns and satisfies all the constraints. %in equation~\ref{eq:constraint_set}.
\end{definition}

\begin{definition}[Non-preferred Trajectory]
A non-preferred trajectory is defined as a trajectory that achieves high returns but violates constraints.
\end{definition}
We assume that  we have access to a limited number of non-preferred trajectories, $\mathcal{D}^N$, and a large number of unlabeled trajectories, $\mathcal{D}^U$, containing both preferred and non-preferred trajectories. Additionally, we represent the empirical distribution of non-preferred trajectories as $\rho^N$. This means that a trajectory  $\tau \in \mathcal{D}^N$ is a sample drawn from $\tau \sim \rho^N$. Likewise, the stationary distribution, $\rho^U$, characterizes the distribution of trajectories within $\mathcal{D}^U$.

\subsection{Multiple Instance Learning}
The goal of the standard binary supervised learning algorithm is to predict the target label, $y \in  \{0, 1\}$, for a given instance, $\text{x}$. In Multiple Instance Learning (MIL) \cite{survey_mil_elsevier_2018} problem, instead of a single instance, there is a bag of instances, $\mathcal{B} = \{\text{x}_1, \text{x}_2, \dots, \text{x}_K\}$, that are neither dependent nor ordered among each other. In general, the size of the bag, $K$, can vary; however, in our work, we chose to fix $K$. We also assume that there exists a binary label for all the instances of the bag, i.e., $y_1, \dots, y_K$ and $y_k \in \{0,1\}\; \forall\; k = 1,\dots, K$, however, we do not have access to these instance labels. Instead, we can access a single binary label, $Y$, associated with each bag. The bag label is determined by the presence or absence of at least one positive instance. A bag is negative, $Y = 0$; when all the instances of the bag are negative, i.e., $y_k = 0,\; \forall k$ and the bag is positive, $Y=1$, when there is at least one positive instance in the bag, i.e., $y_k = 1$ for some $k$. The label $Y$ is defined as:
\begin{equation}
    Y = \begin{cases}
        0,&\; y_k = 0,\; \forall\; k\; \text{in bag } \mathcal{B}\\
        1,&\;\text{otherwise}
    \end{cases}
\end{equation}
MIL has two primary classification tasks: bag-level and instance-level classification. Bag-level classification \cite{bag_emdd_neurips_2001,bag_ensemble_kis_2007} focuses on determining the label of the entire bag of instances, and instance-level classification \cite{instance_miSVM_neurips_2002,instance_amil_icml_2018} requires classifying each instance within the bag.

\section{Methodology}
% This section presents \textit{offline Safe imitation learning via the Multiple Instance Learning} (SafeMIL) algorithm. We learn the cost function by formulating it as a MIL problem, and then we use this learned cost function to learn a safe policy.

This section presents \textit{offline Safe imitation learning via Multiple Instance Learning} (SafeMIL) algorithm. We learn the cost function by formulating it as a MIL problem, and then use it to identify preferred behavior from unlabeled dataset. We learn a safe policy via behavior cloning on these preferred behaviors likely to satisfy CMDP constraints.

\subsection{Formulating cost function learning as MIL problem}
To frame our cost function learning problem as a MIL problem, we need to construct bags and assign positive/negative labels. Suppose we define our bag whose instances are trajectories, $\mathcal{B} = \{\tau_1, \dots, \tau_K\}$. Then, the negative bags can be created by sampling trajectories with replacement from the non-preferred trajectory dataset, as all trajectories within the bag are non-preferred. Similarly, we construct unlabeled bags by sampling trajectories with replacement from unlabeled dataset. For an unlabeled bag to be considered as positive bag, it must contain at least one preferred trajectory. For some bag size $K$, the probability of the unlabeled bag containing at least one preferred trajectory is defined as:
\begin{lemma}
    \label{lm:large_k}
    Let $\mathcal{T}_p$ denote the set of all preferred trajectories. Let $\alpha \in (0,1)$ represent the proportion of preferred trajectories within the unlabeled dataset $\mathcal{D}^U$. Consider a bag $\mathcal{B}$ containing $K$ trajectories sampled with replacement from $\mathcal{D}^U$. Then, the probability that bag $\mathcal{B}$ contains at least one preferred trajectory is given by:
    \begin{align}
        P(\mathcal{B}\; \cap\; \mathcal{T}_p  \neq \emptyset) = 1 - (1-\alpha)^K \nonumber
    \end{align}    
\end{lemma} 
\begin{proof}
    See Appendix A. for the proof. 
\end{proof}

Therefore, for sufficiently large bag size $K$, the probability of an unlabeled bag containing at least one preferred trajectory approaches 1.

We will overload the notation $\rho^N$ to represent both trajectory ($\tau \sim \rho^N$) and negative bag ($\mathcal{B} \sim \rho^N$) sampling, and similarly, overload $\rho^U$ for trajectory ($\tau \sim \rho^U$) and unlabeled bag ($\mathcal{B} \sim \rho^U$) sampling.

As we only have access to bag-level labels, we can define a score for a bag $\mathcal{B}$ such that it can capture the bag's underlying property. Since the trajectories within a bag are unordered and independent, the score should be invariant to any permutation of these trajectories. To ensure this permutation invariance, we adopt a score function based on the Fundamental Theorem of Symmetric Functions with monomials \cite{deepsets_neurips_2017}, defined as:
% As we only have access to bag-level labels, we should define the score for the bag $\mathcal{B}$ such that the score is invariant to the permutation of the trajectories within the bag. This is because the trajectories within a bag are unordered and independent. Therefore, we choose our score function in the specific form of the Fundamental Theorem of Symmetric Functions with monomials \cite{deepsets_neurips_2017}, defined as:
\begin{equation}
    Score(\mathcal{B}) = g\left(\;\sum_{\tau \in \mathcal{B}} f(\tau)\;\right)
\end{equation}
Here, $f$ and $g$ are suitable transformation functions. This score function is invariant to the permutation of the trajectories within the bag. The effectiveness of this score function relies on the appropriate choice of these transformation functions. We opt for intuitive functions for $f$ and $g$. We define the function $f$ parameterized by $\theta$ as: 
\begin{equation}
    % f(\tau) = \dfrac{1}{K} \sum_{(s_t, a_t) \in \tau} \gamma^t \hat{c}_\theta(s_t, a_t)
    f(\tau) = \dfrac{1}{K} \sum_{t=0}^{\Ttau-1} \gamma^t \hat{c}_\theta(s_t, a_t)
\end{equation}
where $(s_t, a_t)$ is the state-action pair at timestep $t$ in the trajectory $\tau$ and $\hat{c}_\theta: \mathcal{S}\times\mathcal{A}\rightarrow (0,1)$ is the parameterized cost function. 
Therefore, the function $f$ estimates the cumulative cost of the trajectory $\tau$. For the function $g$, we choose an identity function. This choice of $g$ allows the bag score to be directly proportional to the sum of the cumulative costs of its trajectories. Given these choices for $f$ and $g$, the score function for the bag $\mathcal{B}$ is:
\begin{align}
    \label{eq:bag_score}
    Score(\mathcal{B}) = \dfrac{1}{K} \sum_{\tau\; \in\; \mathcal{B}} \sum_{t=0}^{\Ttau -1} \gamma^t \hat{c}_\theta(s_t, a_t)
\end{align}
This score function has an intuitive interpretation. As the number of trajectories in the bag, $K$ approaches infinity, the score function converges to the expected cumulative cost of a trajectory in bag $\mathcal{B}$:
\begin{align}
    \lim_{K \rightarrow\infty} Score(\mathcal{B}) &= \mathbb{E}_{\tau \sim \mathcal{B}} \left[ \sum_{t=0}^{\Ttau -1} \gamma^t \hat{c}_\theta(s_t, a_t) \right]
\end{align}
Then, the score function, equation \ref{eq:bag_score}, represents an empirical estimate of the expected cumulative cost of a trajectory associated with the bag $\mathcal{B}$. Assuming non-preferred trajectories have similar cost, then we can establish the following relationship:
\begin{align}
    \label{eq:limit_bags_relationship}
    \lim_{K \rightarrow \infty} Score(\mathcal{B} \sim \rho^N) > \lim_{K\rightarrow\infty} Score(\mathcal{B} \sim \rho^U)
\end{align}    
This means that as the number of trajectories in negative and unlabeled bags approaches infinity, the expected cumulative cost of the negative bag's trajectory will be greater than that of the unlabeled bag. This is based on the intuition that the negative bag sampled from the non-preferred dataset will have more constraint-violating behavior than the unlabeled bag, which, according to Lemma \ref{lm:large_k}, the unlabeled infinite-size bag is likely to contain at least one preferred trajectory.

Similarly, we can explain the relationship between the empirical scores of negative and unlabeled bags:
\begin{theorem}
    \label{th:emperical_score_relation}
    Assuming non-preferred trajectories have similar cost. Then, for some bag size $K$, the score of negative ($\mathcal{B}_n$) and unlabeled ($\mathcal{B}_u$) bag satisfies $Score(\mathcal{B}_n) > Score(\mathcal{B}_u)$ with probability
    \begin{equation}
        P(Score(\mathcal{B}_n) > Score(\mathcal{B}_u)) = 1 -(1-\alpha)^K \nonumber
    \end{equation}
    where $\alpha$ is the percentage of preferred trajectories in the unlabeled dataset.
\end{theorem}
\begin{proof}
    The result follows directly from Lemma \ref{lm:large_k}. The unlabeled bag $\mathcal{B}_u$ contains at least one preferred trajectory with $1 - (1-\alpha)^K$ probability. Therefore, the scores also satisfy the above relationship with the same probability.
\end{proof} 

Given the relationship between the empirical scores of negative and unlabeled bags, we can train the cost function $\hat{c}_\theta$ using Bradley-Terry model \cite{bradley_terry_1952}. Therefore, the loss function can be defined as:
\begin{align}
    \label{eq:cost_learning}
    \mathcal{L}_\theta 
    = - \mathbb{E}_{\substack{\mathcal{B}_n\sim\rho^N,\\\mathcal{B}_u\sim \rho^U}} \left[ \log \dfrac{\exp(Score(\mathcal{B}_n))}{\exp(Score(\mathcal{B}_n)) + \exp(Score(\mathcal{B}_u))} \right]
\end{align}
By minimizing the above loss function, we can to train the cost function $\hat{c}_\theta$, which assigns a higher score to a negative bag than an unlabeled bag. In the next section, we discuss a method that uses this learned cost function to learn a preferred policy.

\subsection{Policy learning}
The learned cost function $\hat{c}_\theta$ estimates whether a given state-action pair is risky, i.e., likely to violate constraints. Since $\mathcal{D}^U$ contains both preferred and non-preferred trajectories, we can use this learned cost function to identify preferred trajectories that are likely to satisfy the CMDP constraints. Specifically, we identify the set of preferred trajectories as:
\begin{align}
    \mathcal{T}_{\hat{c}_\theta} := \left\{\tau \in \mathcal{D}^U \; \Big| \;  \sum_{t=0}^{\Ttau-1} \gamma^t \hat{c}_\theta(s_t, a_t) \leq \hat{b} \right\}    
\end{align}
where $\hat{b}$ is a constraint threshold. We can then recover the preferred policy by optimizing the following loss function:
\begin{align} \label{eq:hard_weighted_bc}
    \min_\pi \sum_{\tau  \mathcal{T}_{\hat{c}_\theta}} \left[ \sum_{t=0}^{\Ttau -1}\mathcal{L}_\pi(s_t, a_t) \right]
\end{align}
where $(s_t, a_t)$ is the state-action pair at timestep $t$ in trajectory $\tau$ and $\mathcal{L}_\pi(\cdot, \cdot)$ is the behavior cloning loss function. 

Alternatively, instead of applying a hard threshold, $\hat{b}$, to select trajectories, we can take a soft weighting approach by assigning weight to each trajectory based on its estimated cost. This results in the following weighted objective:
\begin{align}
    \label{eq:policy_learning}
    \min_\pi \sum_{\tau \in \mathcal{D}^U} \left[w(\tau) \sum_{t=0}^{\Ttau -1}\mathcal{L}_\pi(s_t, a_t)\right]
\end{align}
where the trajectory weight $w(\tau)$ is defined as:
\begin{align}
    \label{eq:unnormalized_weight}
    w(\tau) = \exp \left( -\sum_{t=0}^{\Ttau-1} \gamma^t \hat{c}_\theta(s_t, a_t)\; /\; \beta \right)
\end{align}
and $\beta > 0$ is a hyperparameter. For small value of $\beta$, this weighting scheme assigns higher weight to constraint-satisfying preferred behaviors over non-preferred ones. This approach also allows us to leveraging all available data while prioritizing policy to mimic preferred behaviors while discouraging non-preferred behaviors.

\subsection{Extension to learning cost function with partial trajectories}
The method so far assumes the use of full trajectory length; however, using the entire trajectory can be computationally expensive or sometimes impractical.
Therefore, we modify our algorithm to construct bags with partial trajectories instead of full trajectory lengths. This modification can lead to incorrect sampling of negative bags from non-preferred dataset. The partial trajectories sampled from non-preferred dataset can exhibit preferred behavior, resulting in incorrect labeling of these bags. However, even if some negative bags are mislabeled, the relationship between the scores of the negative bag and unlabeled bag, equation \ref{eq:limit_bags_relationship}, can still hold provided the number of trajectories, $K$, in each bag, is sufficiently large. This is because the score typically reflects the average cost of the trajectories within a bag. Even with some preferred segments in the negative bag, the overall average cost in the negative bag is expected to be higher than that in the unlabeled bag, given a sufficiently large bag size.

Training using partial trajectories allows the SafeMIL algorithm to learn an expressive cost function from only a limited number of non-preferred trajectories. During training, we randomly sample pairs of negative and unlabeled bags, each containing $K$ partial trajectories of length $H$. The score for each bag is then used to compute ``logit" in a binary cross-entropy loss function. The learned cost function is used to identify preferred behavior from unlabeled dataset. We learn a safe policy via behavior cloning on these preferred behaviors likely to satisfy CMDP constraints.
% The policy is learned using a weighted behavior cloning algorithm that uses the learned cost function $\hat{c}_\theta$ to determine the weights assigned to an unlabeled trajectory.
The training procedure is shown in Algorithm \ref{alg:safemil_algorithm}.

\begin{algorithm}[tb]
    \caption{SafeMIL}
    \label{alg:safemil_algorithm}
    \textbf{Input}: non-preferred dataset $\mathcal{D}^N$, unlabeled dataset $\mathcal{D}^U$, bag size $K$, partial trajectory length $H$, $\beta$\\
    \textbf{Parameter}: a parameterized cost function $\hat{c}_\theta$ and a parameterized  policy network $\pi$\\
    \textbf{Output}: preferred policy $\pi \approx \pi^\star$
    \begin{algorithmic}[1] %[1] enables line numbers
        % \STATE Let $t=0$.
        \FOR{$n=1,2,\dots$}
        \STATE Sample bag $\mathcal{B}_n \sim \mathcal{D}^N$ of bag size $K$ and trajectory length $H$
        \STATE Sample bag $\mathcal{B}_u \sim \mathcal{D}^U$ of bag size $K$ and trajectory length $H$
        \STATE Update cost function $\hat{c}_\theta$ by optimizing the Eq. \ref{eq:cost_learning}
        \STATE Update policy $\pi$ on $\mathcal{D}^U$ using the Eq. \ref{eq:policy_learning}
        \ENDFOR
        % \STATE \textbf{return} $\pi$
    \end{algorithmic}
\end{algorithm}

\section{Experiments}
\label{sec:experiments}
\begin{figure*}[t!]
    \centering
    \includegraphics[width=0.7\textwidth]{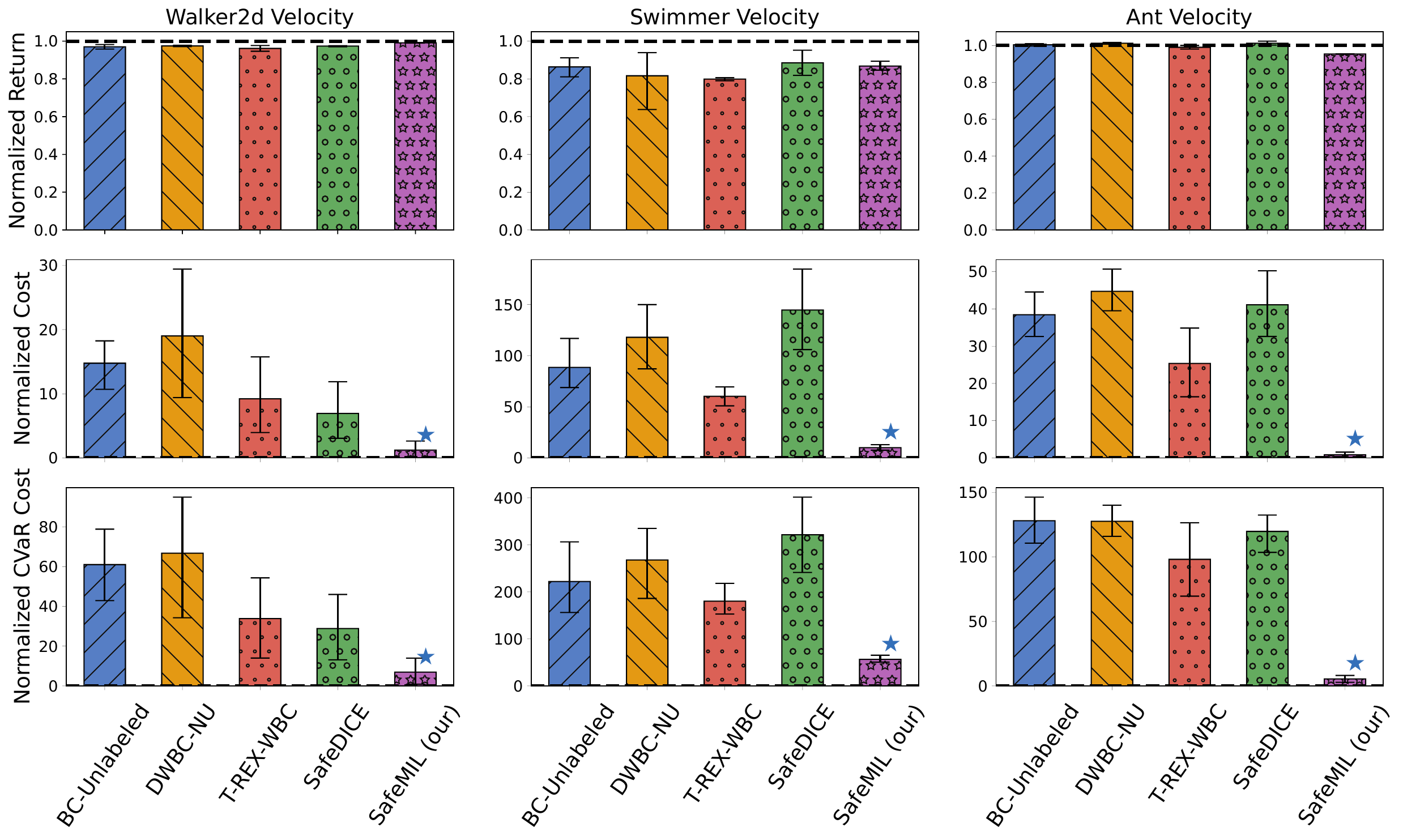}
    \caption{\textbf{Performance Comparison.} We report the final bootstrapped mean performance of the algorithm on Walker2d-Velocity, Swimmer-Velocity, Ant-Velocity task after 1 million training steps. Mean and 95\% CIs over 5 seeds. We observe that our method outperforms all the baselines and can recover low cost safe policies without compromising reward performance. We also report the learning curves for all the algorithms in Fig. \saferef{fig:norm_velocity_task_performance}{8} in Appendix.}
    \label{fig:velocity_task_performance}
\end{figure*}

\begin{figure*}[htp]
    \centering
    \includegraphics[width=0.7\textwidth]{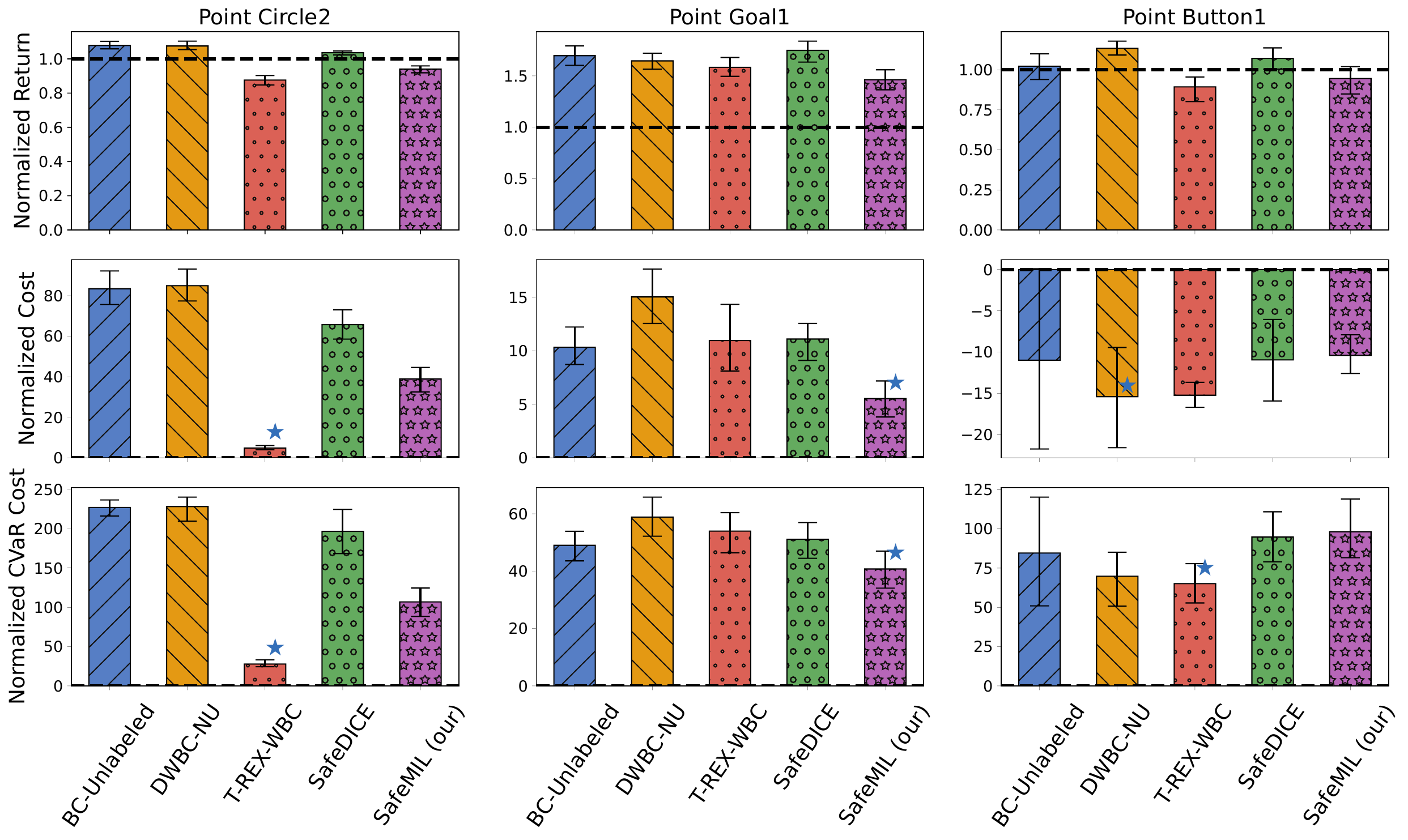}
    \caption{\textbf{Performance Comparison.} We report the final bootstrapped mean performance of the algorithm on Point-Circle2, Point-Goal1, Point-Button1 tasks after 1 million training steps. Mean and 95\% CIs over 5 seeds. We observe that the Point-Goal1 environment, our algorithm performs better, while maintaining competitive performance in the remaining environments. We also report the learning curves for all the algorithms in Fig. \saferef{fig:norm_point_task_performance}{9} in Appendix.}
    \label{fig:point_task_performance}
\end{figure*}

We evaluate our algorithm against the state-of-the-art offline safe IL algorithm using the  Datasets for offline Safe
RL (DSRL) on a set of benchmark tasks  \cite{offline_safe_rl_dataset_jmlr_2024}. We conduct experiments on the following tasks: i) MuJoCo-based velocity-constrained tasks (Walker-Velocity, Swimmer-Velocity, Ant-Velocity), where agents are required to move as fast as possible while adhering to the velocity limits; and ii) navigation tasks (Point-Circle2, Point-Goal1, Point-Button1), that require agents to maximize performance while avoiding collisions or contact with hazardous regions (illustrated in Appendix Fig: \saferef{fig:velocity_task_env}{5}, \saferef{fig:point_task_env}{6}). These tasks collectively provide comprehensive and realistic scenarios for evaluating the safety and the performance of offline safe IL algorithms. 

We select a limited number of non-preferred trajectories ($\mathcal{D}^N$) and a large number of unlabeled trajectories ($\mathcal{D}^U$) from the DSRL dataset. $\mathcal{D}^U$ contain a mix of preferred and non-preferred trajectories. We removed all the reward and cost information from $\mathcal{D}^N$ and $\mathcal{D}^U$ trajectory dataset. Our objective is to recover low-cost behavior while preserving high-reward performance from the unlabeled data. We have 50 trajectories in $\mathcal{D}^N$ and 200 trajectories in $\mathcal{D}^U$. In Appendix D., we list all the other details regarding $\mathcal{D}^N$ and $\mathcal{D}^U$ dataset for each environment.

\paragraph{Baselines:} We compare our algorithm against four baselines. \textbf{(1) BC-Unlabeled}, a BC policy on the unlabeled dataset $\mathcal{D}^U$. This serves as a baseline to evaluate the performance of BC when the dataset contains a mixture of preferred and non-preferred trajectories. \textbf{(2) SafeDICE}, \citet{safedice_neurips_2023}, directly estimates the stationary distribution corrections for the preferred behavior and then trains a weighted BC policy. \textbf{(3) DWBC-NU}, a variant of the discriminator-weighted behavior cloning algorithm \cite{dwbc_icml_2022} that uses \textit{negative-unlabeled} learning to train the discriminator model. The trained discriminator is then used as weight in the weighted BC loss function. \textbf{(4) T-REX-WBC}, \citet{trex_icml_2019}, a preference-based method that learns a reward function to prefer unlabeled trajectories over non-preferred ones. The learned reward function is then used as a weight in the weighted BC loss function. The implementation details of these baseline algorithms are in the Appendix C.

We also train an offline constrained RL policy, COptiDICE \cite{coptidice_iclr_2022}, on an unlabeled dataset augmented with ground-truth reward and cost annotations for 1 million training steps. Constrained-RL policy represents the optimal performance achievable by a policy that simultaneously achieves high reward and low cost. We treat this Constrained-RL policy as the preferred low-cost policy. No other methods have access to reward and cost information. We do not compare our algorithm with other constrained RL algorithms, as our setting, offline safe IL, fundamentally differs from offline safe RL.

We compare the performance of SafeMIL and other baseline algorithms in terms of task performance (i.e., expected episodic return) and safety (i.e., expected episodic cost). We report the results with the following metrics. \textbf{(1) Normalized Return} scales the episodic return of a given policy, such that $0$ represents the episodic return from a random policy, while $1$ represents the episodic return achieved by the Constrained-RL policy. \textbf{(2) Normalized Cost} scales a given policy's episodic cost by subtracting the Constrained-RL policy's episodic cost. It ensures that $0$ represents the cost achieved by the Constrained-RL policy. \textbf{(3) Normalized Conditional Value at Risk performance (CVaR) 20\% Cost}, similar to \textit{Normalized Cost} it scales the policy's mean episodic cost of the worst 20\% runs by subtracting the Constrained-RL policy's episodic cost.

Our primary focus of the comparison is to evaluate safety (\textit{Normalized Cost}, \textit{Normalized CVaR@20\% Cost}). 
% We use \textit{Normalized Return} to examine whether the agents exhibit abnormal behaviors.
We consider a policy to have successfully recovered the preferred behavior if its performance closely matches the Constrained-RL policy. This is evidenced by the \textit{Normalized Cost} less than or close to $0$, and the \textit{Normalized Return} being greater than or close to $1$. All plots are generated by averaging the performance of 50 trajectories generated from the learned policy. To assess statistical significance, we generate 1000 bootstrap samples from the data, using results from 5 different random seeds, and we plot the resulting 95\% confidence intervals.

Through our experiments, we focus on answering the following questions:
\begin{enumerate}
    \item \textit{Performance Comparison}: How does our algorithm compare to other baselines in terms of learning safer policies that satisfy cost constraints while maintaining reward performance?
    \item {Sensitivity to the choice of bag size}: How does our algorithm's performance vary with different bag sizes $K$?
    \item {Sensitivity to the choice of trajectory length}: What is the impact of trajectory length on our algorithm's performance?
    \item {Effect of different weighting scheme in policy learning}: How does our algorithm's performance change if we define weights on individual state-action pairs instead of the entire trajectory when learning the policy?
\end{enumerate}

\begin{figure*}[t]
    \centering
    \includegraphics[width=0.8\textwidth]{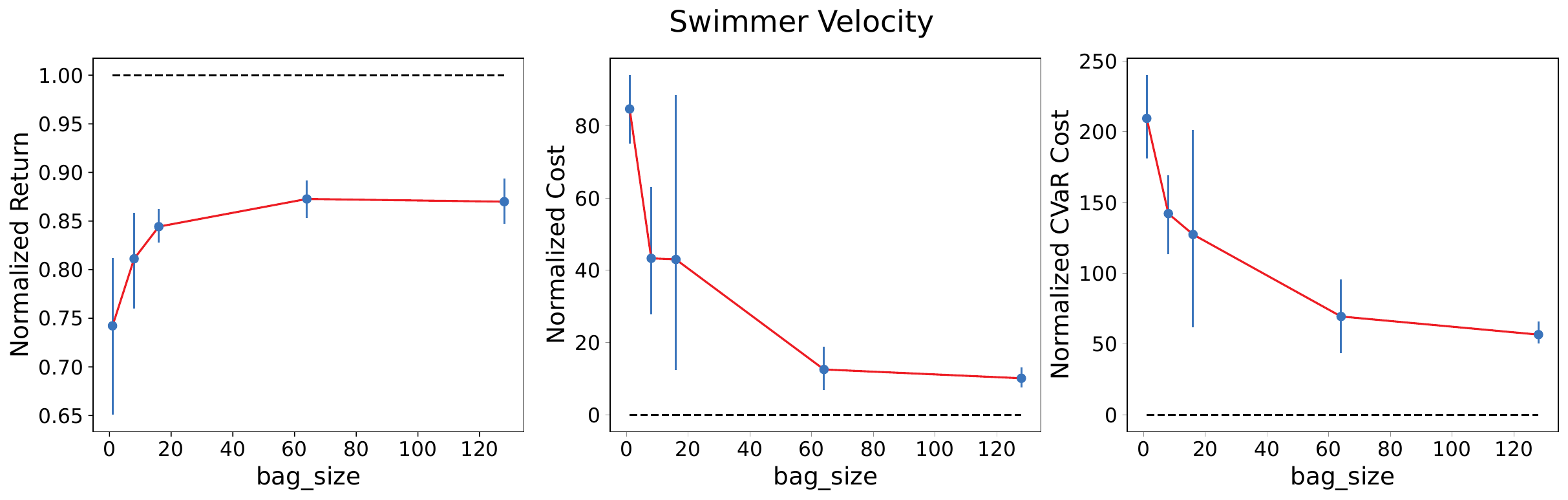}
    \caption{\textbf{Sensitivity to Bag Size.} We report the final mean performance of the algorithm on the Swimmer-Velocity environment for different bag sizes $K=\{1, 8, 16, 64, 128\}$, after training for 1 million steps. Mean and 95\% CIs over 5 seeds. We observe that increasing the bag size $(K)$ lead to a higher probability of finding preferred trajectories within the unlabeled bag. Thereby, improving the algorithm's safety performance while maintaining reasonable episode return.}
    \label{fig:swimmer_sensitivity_bag_size}
\end{figure*}

\begin{figure*}[htp]
    \centering
    \includegraphics[width=0.8\textwidth]{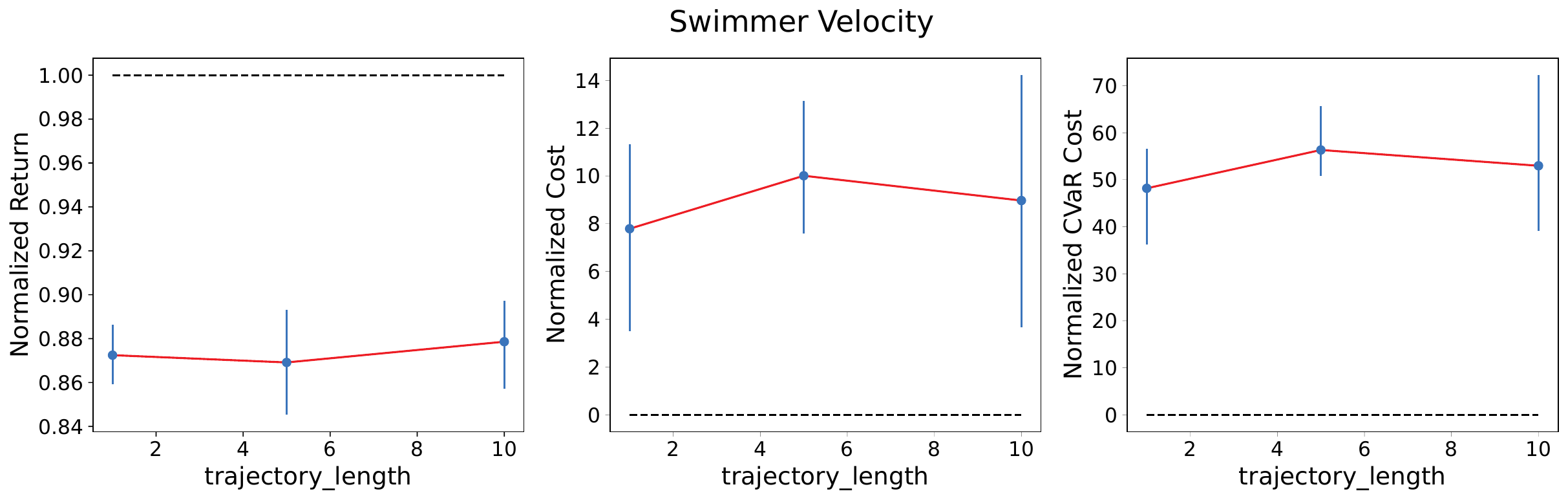}
    \caption{\textbf{Sensitivity to Trajectory Length.} We report the final mean performance of the algorithm on the Swimmer-Velocity environment for different partial trajectory lengths $H=\{1, 5, 10\}$, after training for 1 million steps. Mean and 95\% CIs over 5 seeds. For a sufficient bag size of $128$, we observe that safety performance is stable across different trajectory lengths.}
    \label{fig:swimmer_sensitivity_trajectory_length}
\end{figure*}

\subsection{Performance Comparison}
In both velocity-constrained tasks (Fig. \ref{fig:velocity_task_performance}) and navigation tasks (Fig. \ref{fig:point_task_performance}), we verify that when the unlabeled dataset contains non-preferred (constraint-violating) trajectories, the standard BC policy has a higher cost and therefore can exhibit non-preferred behavior.

SafeMIL can outperform all baseline algorithms in velocity-constrained tasks. It is also able to recover preferred behavior from the unlabeled dataset, with safety performance closely approximating that of Constrained-RL. SafeMIL learns a safer policy for navigation tasks in Point-Goal1 and is competitive compared to other baselines in Point-Circle2 and Point-Button1. In Point-Circle2, T-REX-WBC learns a policy with a better constraint-satisfying policy. SafeMIL achieves a median performance of $3.7\times$ better than the best baseline algorithm across all environments. In Appendix F., we report other supporting results and the learning curves for all the algorithms.
% While T-REX-WBC and SafeDICE can recover the preferred behavior for the Walker2d velocity task, all baselines fail to recover the preferred behavior for the Swimmer and the Ant velocity tasks.
% T-REX-WBC learns the policy that prioritizes unlabeled trajectories over non-preferred trajectories. As a result, it can learn sub-optimal behavior since non-preferred trajectories are also present in $\mathcal{D}^U$. 
% In Point-Circle2, T-REX learns a policy with better constraint satisfying policy at the expense of lower episode return.

\subsection{Sensitivity Analysis}
This section investigates how the bag's size and the trajectory length impact our algorithm's performance. We demonstrate that increasing the bag size and using partial trajectory length does not alter the algorithm's safety or the overall reward achieved.

\paragraph{\textit{Sensitivity to Bag Size}:}
To assess the impact of bag size on our SafeMIL algorithm, we conducted experiments on the Swimmer-Velocity environment using varying bag sizes, $K = \{1, 8, 16, 64, 128\}$. We train the SafeMIL algorithm with these different bag sizes for 1 million training steps and report the final performance in  Fig. \ref{fig:swimmer_sensitivity_bag_size}. Based on Theorem \ref{th:emperical_score_relation}, we expect that increasing the bag size $(K)$ would lead to a higher probability of finding preferred trajectories within the unlabeled bag. Therefore, increasing bag size would improve the algorithm's safety performance  while maintaining reasonable episode return. As shown in Fig. \ref{fig:swimmer_sensitivity_bag_size}, as $K$ increases, the safety performance also improves (lower cost) and becomes stable after some bag size. 
% We also report the safety performance for Point-Goal1 task (in supplementary material Figure: \ref{fig:goal_sensitivity_bag_size}). 

\paragraph{\textit{Sensitivity to Trajectory Length}:}
To evaluate the effect of trajectory length on our SafeMIL algorithm, we performed experiments on the Swimmer-Velocity environments for sufficiently large bag size ($K = 128$) with varying trajectory length, $H = \{1, 5, 10\}$. The SafeMIL algorithm is trained with these trajectory lengths for 1 million training steps, and the final performance results are summarized in Fig. \ref{fig:swimmer_sensitivity_trajectory_length}. We expect that for sufficiently large bag size, the relationship between the scores of the negative bag and unlabeled bag, equation \ref{eq:limit_bags_relationship}, will still hold. Therefore, the safety performance of the task should not vary. As shown in Fig. \ref{fig:swimmer_sensitivity_trajectory_length}, safety performance is stable irrespective of the trajectory length. This finding supports using partial trajectories for training, as they offer a computationally efficient alternative to utilizing full-length trajectories without compromising overall performance. 
% We observe similar safety performance for Point-Goal1 task (in supplementary material Figure: \ref{fig:goal_sensitivity_trajectory_length}).

\subsection{Effect of different weighing scheme}
This section analyzes the effect of different weighing scheme in policy learning. Specifically, we want to verify the algorithm's performance when we define weights on individual state-action pair transitions instead of the entire trajectory when learning the policy. Since $\hat{c}_\theta$ estimates the likelihood of a state-action pair being risky, then we can learn the safe policy by defining the weight in the BC loss function at each state-action $(s, a)$ pair as: $(1 - \hat{c}_\theta(s, a))$. The loss function is then defined as:
\begin{align}
    % \min_\pi \mathbb{E}_{(s,a) \sim \rho^U} \left[ (1-\hat{c}_\theta(s,a)) \mathcal{L}_\pi(s,a)\right]
    \min_\pi \sum_{(s,a) \in \mathcal{D}^U} (1-\hat{c}_\theta(s,a)) \mathcal{L}_\pi(s,a)
\end{align}
We refer to this loss function as the transition weighted BC loss function and call equation \ref{eq:policy_learning} as trajectory weighted BC loss function. In Appendix, Fig. \saferef{fig:ablation_study}{12}, \saferef{fig:ablation_study_goal1}{13} reports the result of using transition and trajectory weighted BC loss function to train safe policy. We observe that both weighing schemes for policy learning have similar safety and return performance for Swimmer-Velocity and Point-Goal1 tasks. 

\section{Conclusion}

This work introduces SafeMIL, an offline safe imitation learning algorithm. We frame the problem of learning a cost function as an MIL task. Based on the intuition that non-preferred trajectories incur higher costs than unlabeled ones on average, we learn the cost function that explains this intuition. Subsequently, we employ this learned cost function to identify preferred behavior and use it to learn safe policy.
% construct the weights in the weighted BC algorithm.
The experiment demonstrates that SafeMIL can learn safer policies in the constrained RL benchmarks. % Furthermore, we empirically show that SafeMIL is not sensitive to sufficient large bag size and partial trajectory length. 

\section{Acknowledgments}
Returaj Burnwal acknowledges financial support from the Prime Minister’s Research Fellowship (PMRF), Ministry of Education, India.

% \bibliography{aaai2026}

\clearpage

\addcontentsline{toc}{chapter}{Suppl}

\setcounter{section}{0}
\renewcommand{\thesection}{\Alph{section}}
\onecolumn

\section{A.\; Proof of Lemma 1}
\label{sup:lm_1_proof}
\textbf{Lemma 1.} \textit{
    Let $\mathcal{T}_p$ denote the set of all preferred trajectories. Let $\alpha \in (0,1)$ represent the proportion of preferred trajectories within the unlabeled dataset $\mathcal{D}^U$. Consider a bag $\mathcal{B}$ containing $K$ trajectories sampled with replacement from $\mathcal{D}^U$. Then, the probability that bag $\mathcal{B}$ contains at least one preferred trajectory is given by:
    \begin{align}
        P(\mathcal{B}\; \cap\; \mathcal{T}_p  \neq \emptyset) = 1 - (1-\alpha)^K \nonumber
    \end{align}    
}

\begin{proof}
    Given that $\alpha$ represents the proportion of preferred trajectories, $(1-\alpha)$ naturally represents the proportion of non-preferred trajectories in unlabeled dataset $\mathcal{D}^U$. Therefore, the probability that a bag $\mathcal{B}$ contains only non-preferred trajectories is $(1-\alpha)^K$. Then the probability of the complementary event i.e., the bag containing at least one preferred trajectory is $1-(1-\alpha)^K$.
\end{proof} 
\section{B.\; Constrained RL Environments}
We conducted our experiments on the following environments.

\begin{figure*}[h!]
    \centering
    \includegraphics[width=0.9\textwidth]{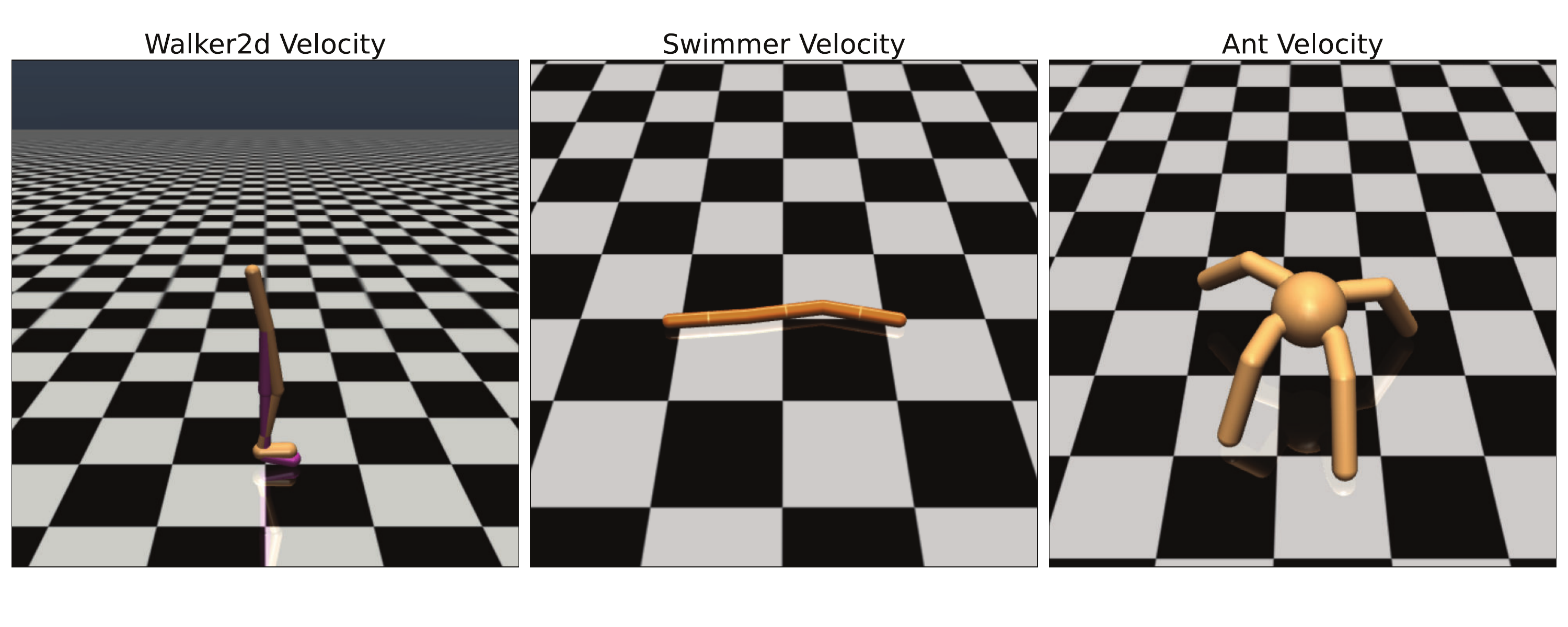}
    \caption{MuJoCo-based velocity-constrained tasks of DSRL environment. For each task, the agent needs to move as fast as possible while adhering to the velocity limits.}
    \label{fig:velocity_task_env} 
\end{figure*}

\begin{figure*}[h!]
    \centering
    \large 
    \includegraphics[width=0.9\textwidth]{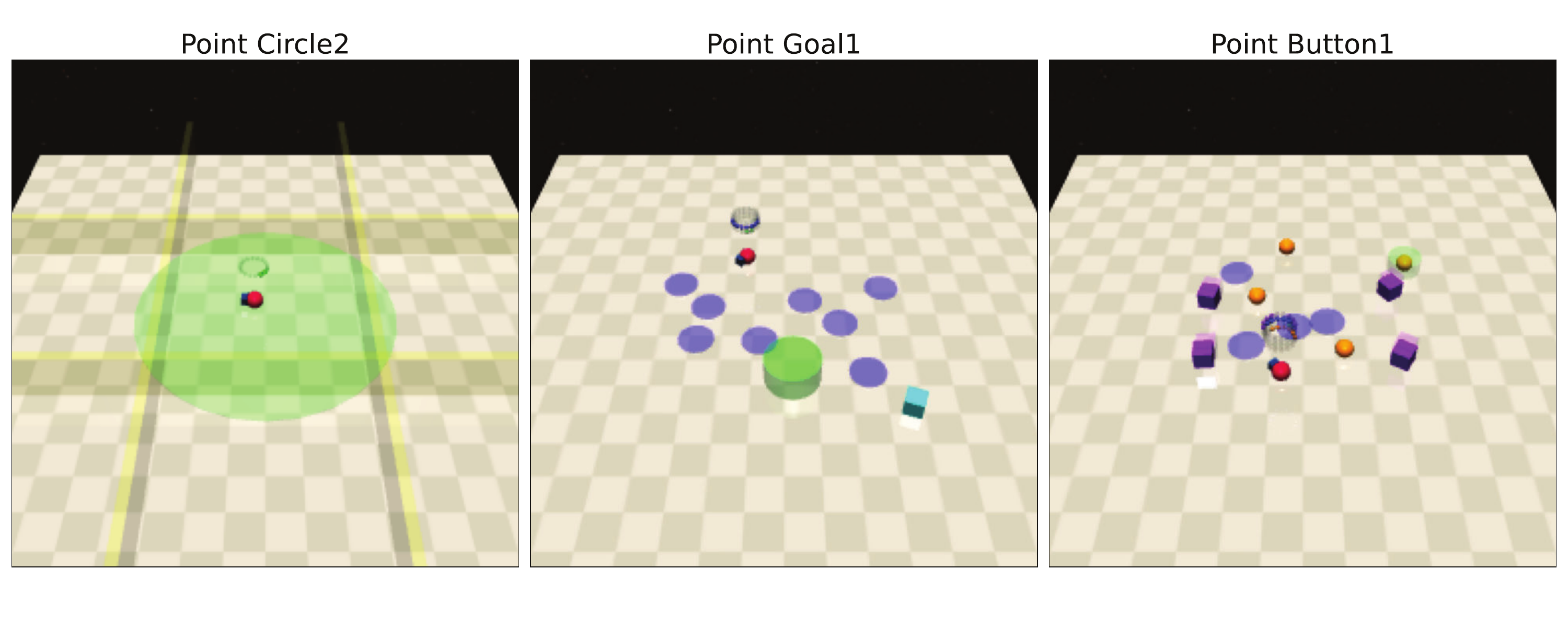}
    \caption{Navigation tasks of DSRL environment. For each task, the agent (red robot) receive costs when crossing the yellow boundaries, entering hazards indicated by blue circles or when touching the purple obstacles. The objectives of each task are as follows: \textbf{Point Circle2}: agent needs to circle around the center area as close to the boundaries for optimal reward, \textbf{Point Goal1}: Move to a series of green goal positions, \textbf{Point Button1}: Press a series of highlighted goal buttons.}
    \label{fig:point_task_env}
\end{figure*}

% \begin{figure*}[h!]
%     \centering
%     \begin{subfigure}[b]{0.9\textwidth}
%         \includesvg[width=\linewidth]{images/svg_images/env/velocity_task_env.svg}
%         \caption{\large MuJoCo-based velocity-constrained tasks of DSRL environment. For each task, the agent needs to move as fast as possible while adhering to the velocity limits.}
%         \label{fig:velocity_task_env} 
%     \end{subfigure}
%     % \vspace{1pc}
%     \begin{subfigure}[b]{0.9\textwidth}
%         \includesvg[width=\linewidth]{images/svg_images/env/point_task_env.svg}
%         \caption{\large Navigation tasks of DSRL environment. For each task, the agent (red robot) receive costs when crossing the yellow boundaries, entering hazards indicated by blue circles or when touching the purple obstacles. The objectives of each task are as follows: \textbf{Point Circle2}: agent needs to circle around the center area as close to the boundaries for optimal reward, \textbf{Point Goal1}: Move to a series of green goal positions, \textbf{Point Button1}: Press a series of highlighted goal buttons.}
%         \label{fig:point_task_env}
%     \end{subfigure}
%     \label{fig:dsrl_env}
% \end{figure*}

\section{C.\; Implementation Details of the baseline algorithms}
\label{sup:implementation_details}
In this section, we provide the implementation details of the baseline algorithms.

\subsection{SafeDICE}
SafeDICE algorithm first estimates the log ratio of $\rho^N(s,a)$ and $\rho^U(s,a)$ by training a discriminator model:
\begin{align*}
    c^* = \text{arg}\max_c \mathbb{E}_{(s,a)\sim\rho^N}[\log c(s,a)] + \mathbb{E}_{(s,a)\sim\rho^U}[\log (1-c(s,a))]
\end{align*}
The discriminator model $c^*$ is then used to compute the log ratio:
\begin{align*}
    \text{log ratio} = r_\alpha(s,a) = \log \dfrac{1-(1+\alpha)c^*(s,a)}{(1-\alpha)(1-c^*(s,a))}
\end{align*}
The log ratio estimate is then used to estimate $\nu$ network:
\begin{align*}
    \min_\nu (1-\gamma)\mathbb{E}_{s\sim\rho_0}[\nu(s)] + \log\mathbb{E}_{(s,a,s')\sim\rho^U}\left[\exp(A_\nu(s,a,s'))\right]
\end{align*}
where $A_\nu(s,a,s') = r(s,a) + \gamma \nu(s')$. Then, the safe policy is estimated as follows:
\begin{align*}
    \min_\pi - \dfrac{\mathbb{E}_{(s,a,s') \sim \rho^U}[w_\nu(s,a,s')\mathcal{L}_\pi(s,a)]}{\mathbb{E}_{(s,a,s')\sim\rho^U}[w_\nu(s,a,s')]}
\end{align*}
where $w_\nu(s,a,s') = \exp(A_{\nu^*}(s,a,s'))$ and $\mathcal{L}_\pi$ denotes the BC loss function.

\subsection{T-REX-WBC}
We train the reward function based on  Bradley-Terry model \cite{bradley_terry_1952}, which is as follows:
\begin{align*}
    \min_r - \sum_{(i,j)\in \mathcal{P}} \log \dfrac{\exp\left( \sum_{(s,a)\in\tau_i} r(s,a) \right)}{\exp\left( \sum_{(s,a)\in\tau_i} r(s,a) \right) + \exp\left( \sum_{(s,a)\in\tau_j} r(s,a) \right)}
\end{align*}
where $\tau_i$ and $\tau_j$ are trajectories and $\mathcal{P} = \{(i,j): \tau_i \succ \tau_j\}$. In our setting we assumed that unlabeled demonstrations are better than non-preferred demonstrations, i.e. $\tau_i \succ \tau_j \;\forall\; \tau_i \sim \rho^U, \tau_j \sim \rho^N$. Based on this reward function we train our safe policy as follows:
\begin{align*}
    \min_\pi \mathbb{E}_{(s,a)\sim\rho^U}[r(s,a)\mathcal{L}_\pi(s,a)]
\end{align*}

\subsection{DWBC-NU}
We modified DWBC and used negative-unlabeled learning to train the discriminator model as follows:
\begin{align*}
    \min_d\; &\eta \mathbb{E}_{(s,a)\sim\rho^N} [-\log d(s,a, \log \pi)] \\
    &+ \mathbb{E}_{(s,a)\sim\rho^U}[-\log(1-d(s,a,\log \pi))] \\
    &- \eta\mathbb{E}_{(s,a)\sim\rho^N}[-\log(1-d(s,a,\log\pi))]
\end{align*}
where $d$ is the discriminator model and $\eta$ is the hyperparameter. We then train safe policy by using this discriminator model to construct weight in the weighted BC loss function as follows:
\begin{align*}
    \min_\pi \mathbb{E}_{(s,a)\sim\rho^U}[(1-d(s,a))\mathcal{L}_\pi(s,a)]
\end{align*}

\subsection{COptiDICE}
An offline constrained-RL algorithm, COptiDICE \cite{coptidice_iclr_2022}, requires access to both cost and reward annotation to train. We use COptiDICE to compare our learned policies performance. We use the COptiDICE implementation from OSRL library \cite{offline_safe_rl_dataset_jmlr_2024}. Github link: \url{https://github.com/liuzuxin/OSRL/tree/main}
\section{D.\; Dataset Related Details}
\label{sup:dataset_related_details}
To evaluate our approach, we focused on the velocity-constrained and navigation tasks from the DSRL dataset \cite{offline_safe_rl_dataset_jmlr_2024}. We define preferred and non-preferred trajectories within the dataset based on performance criteria: trajectories with above 50\% total reward and within the top 25\% of total cost are classified as non-preferred, while those in the bottom 25\% of total cost and above 50\% in total reward are considered as preferred (illustrated in Figure \ref{fig:dsrl_dataset}). Using these preferred and non-preferred trajectories we created two datasets: 1) non-preferred dataset: comprising 50 non-preferred trajectories and 2) unlabeled dataset: containing 200 trajectories, a mixture of both preferred and non-preferred trajectories. Table 1 summarizes the tasks, safety constraints, and dataset details for each domain used in our experiments.

% We conducted experiments on the velocity-constrained and navigation tasks from DSRL \cite{offline_safe_rl_dataset_jmlr_2024} tasks. From DSRL dataset we selected those trajectories with above 50\% in reward and with top 25\% of highest total cost as non-preferred trajectories and least 25\% in total cost as preferred trajectories. Using these preferred and non-preferred trajectories we created two datasets: i) non-preferred dataset which contain 50 non-preferred trajectories and ii) unlabeled dataset which contains 100 trajectories with both preferred and non-preferred trajectories. Table \ref{table:task_specification} shows the tasks, safety constraints, and the information of the dataset for each domain that we used in our experiments. 

\begin{figure*}[!htb]
    \centering
    \includegraphics[width=0.8\textwidth]{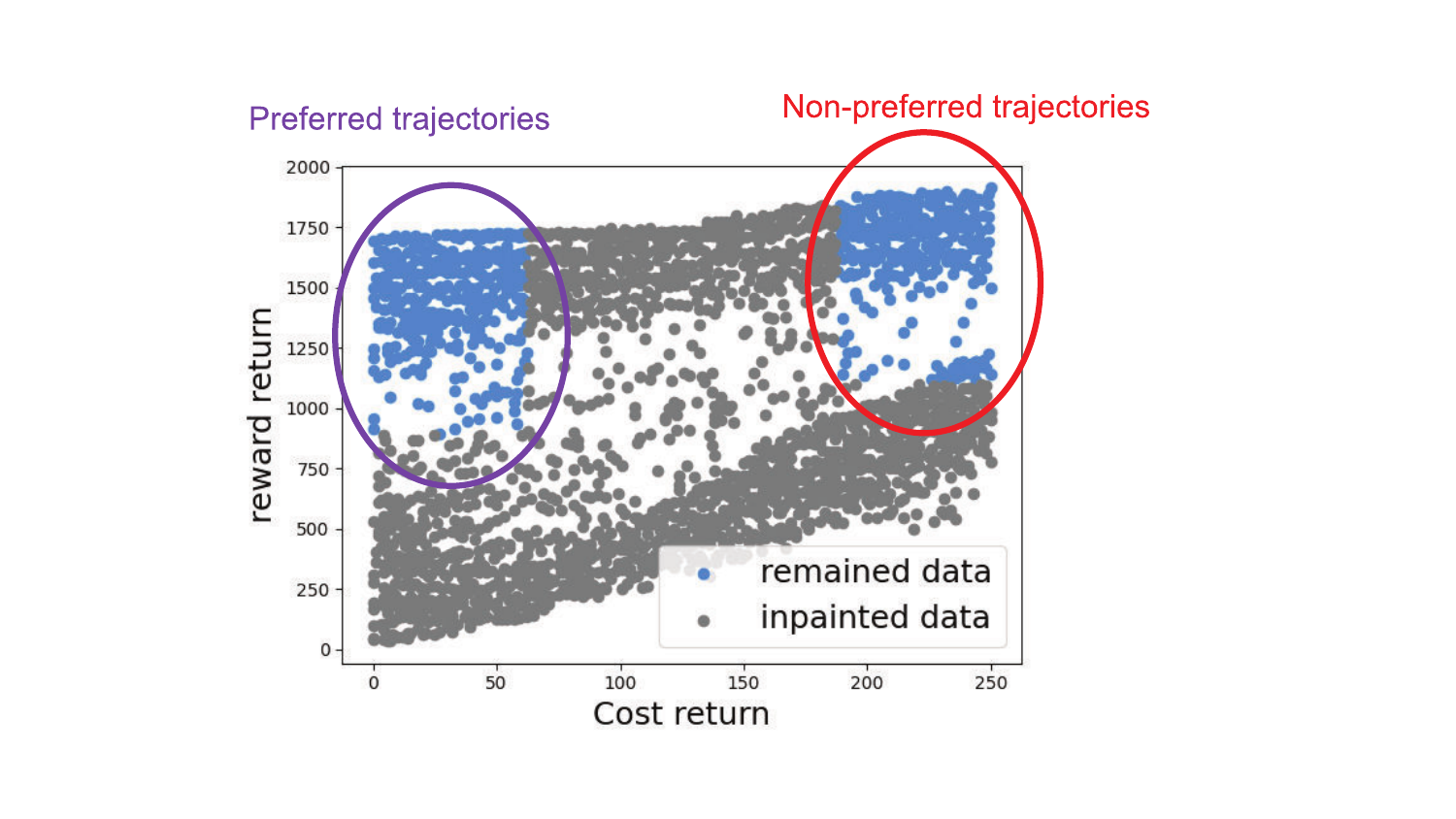}
    \caption{Defining preferred and non-preferred trajectories in DSRL dataset. With upper left corner as preferred and upper right corner as non-preferred trajectories.}
    \label{fig:dsrl_dataset}
\end{figure*}

\begin{table}[!htb]
\centering
\addtolength{\tabcolsep}{-3pt}
\def\arraystretch{1.5}
\begin{tabular}{@{}lcccccc@{}}
\toprule
                                            & \multicolumn{3}{c}{Velocity Constrained} & \multicolumn{3}{c}{Navigation} \\ \midrule\midrule
Task Specification                          & Velocity     & Velocity    & Velocity    & Circle    & Goal    & Button   \\
Type of agent                               & Walker2d     & Swimmer     & Ant         & Point     & Point   & Point    \\
Difficulty Level                            & -            & -           & -           & 2         & 1       & 1        \\ \midrule
\# unlabeled demonstrations                 & 200          & 200         & 200         & 200       & 200     & 200      \\
\# non-preferred demonstrations             & 50           & 50          & 50          & 50        & 50      & 50       \\ \midrule
Mean cost of preferred demonstrations       & 12.08        & 9.94        & 16.46       & 11.97     & 11.96   & 12.47    \\
Mean cost of non-preferred demonstrations   & 324.70       & 166.64      & 221.30      & 255.14    & 88.0    & 174.18   \\
Mean return of preferred demonstrations     & 2501.99      & 128.99      & 1864.79     & 39.17     & 22.70   & 27.67    \\
Mean return of non-preferred demonstrations & 2941.42      & 209.89      & 2864.90     & 50.20     & 22.97   & 30.10    \\ \bottomrule
\end{tabular}
\fontsize{9.0pt}{10.25pt}\selectfont
\captionsetup{justification=centering}
\caption{Task specification of each domain used in our experimental results}
\label{table:task_specification}
\end{table}

\section{E.\; Hyperparameter configurations}
For fair comparison, we use the same architecture and learning rate to train the policy, discriminator, reward, and  cost models of each algorithm. Table \ref{table:hyperparameter} summarizes the hyperparameter configurations that we used in our experiments.
We use the same hyperparameters throughout our experiments, except where explicitly stated.

\begin{table}[!htb]
\centering
\addtolength{\tabcolsep}{-3pt}
\def\arraystretch{1.5}
\begin{tabular}{@{}lccccc@{}}
\toprule
Hyperparameters                              & BC-Unlabeled      & DWBC-NU           & T-REX-WBC         & SafeDICE          & SafeMIL (ours)                                                                                 \\ \midrule \midrule
$\gamma$ (discount factor)                   & 0.99              & 0.99              & 0.99              & 0.99              & 0.99                                                                                           \\
learning rate (actor)                        & $1\times 10^{-5}$ & $1\times 10^{-5}$ & $1\times 10^{-5}$ & $1\times 10^{-5}$ & $1\times 10^{-5}$                                                                              \\
network size (actor)                         & $[256, 256]$      & $[256, 256]$      & $[256, 256]$      & $[256, 256]$      & $[256, 256]$                                                                                   \\
learning rate (cost)                         & -                 & -                 & -                 & $1\times 10^{-5}$ & $1\times 10^{-5}$                                                                              \\
network size (cost)                          & -                 & -                 & -                 & $[256, 256]$      & $[50, 256, 256]$                                                                               \\
learning rate (discriminator / reward model) & -                 & $1\times 10^{-5}$ & $1\times 10^{-5}$ & $1\times 10^{-5}$ & -                                                                                              \\
network size (discriminator / reward model)  & -                 & $[256, 256]$      & $[256, 256]$      & $[256, 256]$      & -                                                                                              \\
gradient penalty coefficient                 & -                 & -                 & -                 & 10                & -                                                                                              \\
weight decay                                 & 0.01              & 0.01              & 0.01              & 0.01              & 0.01                                                                                           \\
$\eta$                                       & -                 & 0.5               & -                 & -                 & -                                                                                              \\
$\beta$                                      & -                 & -                 & -                 & -                 & 0.5                                                                                            \\ \midrule\midrule
bag size                                     & -                 & -                 & -                 & -                 & 128                                                                                            \\
trajectory length                            & -                 & -                 & 5                 & -                 & \begin{tabular}[c]{@{}l@{}}5 (Velocity Tasks)\\ 10 (Navigation Tasks)\end{tabular} \\
batch size                                   & 128               & 128               & 128               & 128               & 32                                                                                             \\
\# training steps                            & $1,000,000$       & $1,000,000$       & $1,000,000$       & $1,000,000$       & $1,000,000$                                                                                    \\ \bottomrule
\end{tabular}
\fontsize{9.0pt}{10.25pt}\selectfont
\captionsetup{justification=centering}
\caption{Hyperparameters used in our experimental results}
\label{table:hyperparameter}
\end{table}

\clearpage

\section{F.\; Additional results}
\subsection{Performance Comparison Results}
\label{sup:performance_comparison}

We report the learning curves of all the algorithms for 1 million timestep.
\begin{figure*}[!htb]
    \centering
    \includegraphics[width=0.9\textwidth]{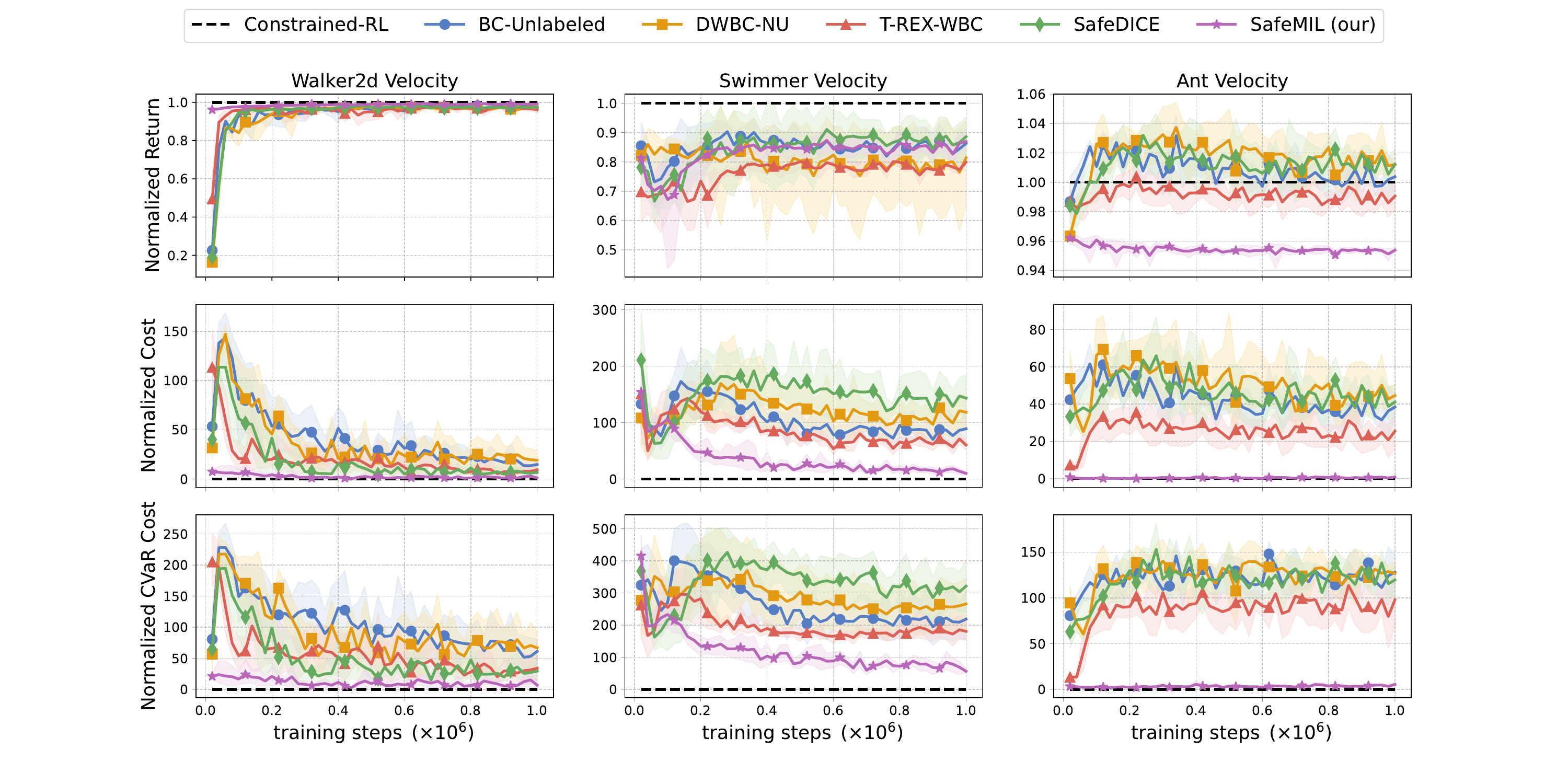}
    \caption{\textbf{Performance Comparison.} Experimental results on Walker2d-Velocity, Swimmer-Velocity, Ant-Velocity task. Shaded  area represents the standard error. In velocity-constrained tasks, our method is able to recover safer policies without compromising reward performance.}
    \label{fig:norm_velocity_task_performance}
\end{figure*}

\begin{figure*}[!htb]
    \centering
    \includegraphics[width=0.9\textwidth]{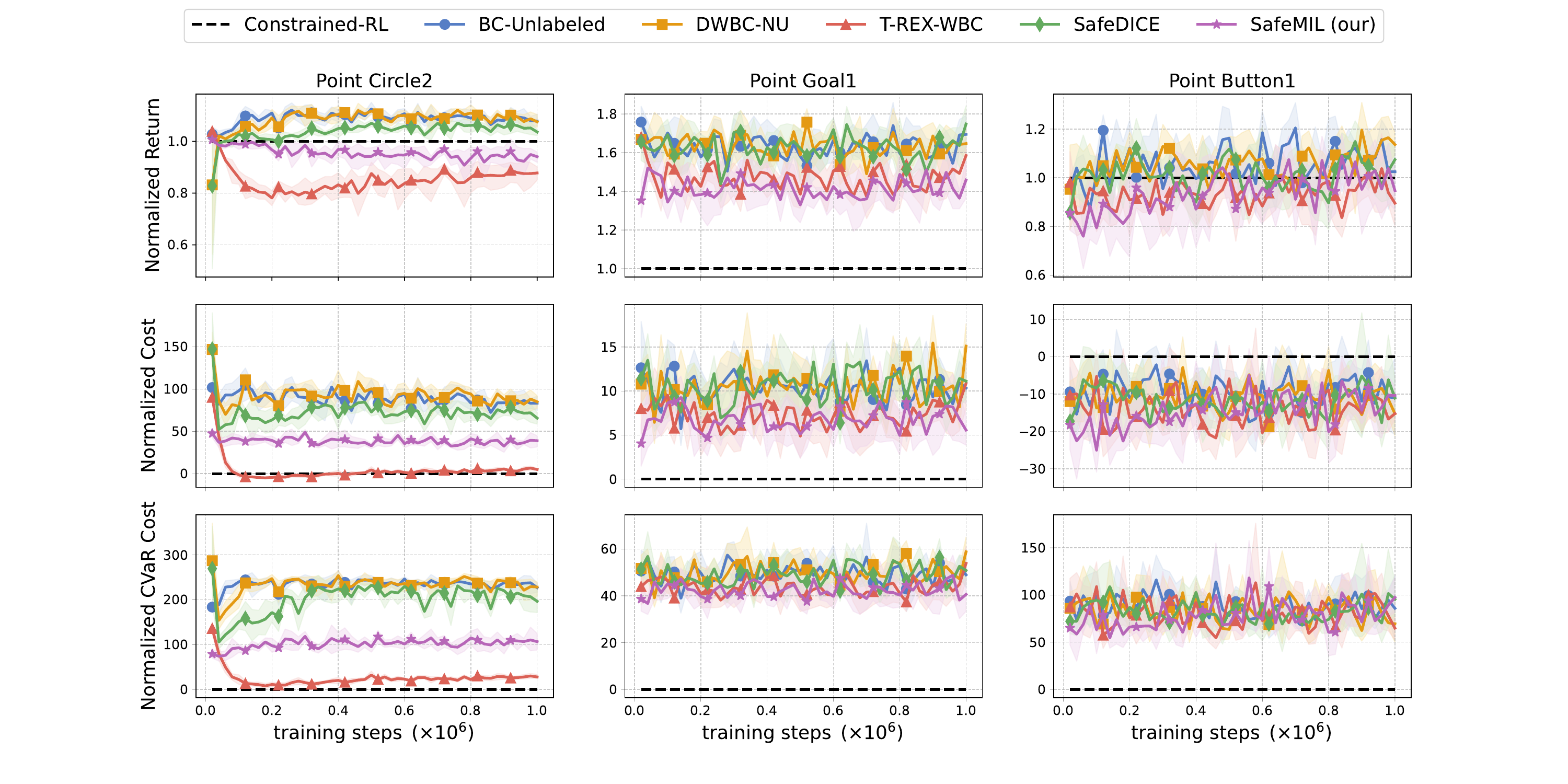}
    \caption{\textbf{Performance Comparison.} Experimental results on Point-Circle2, Point-Goal1, Point-Button1 tasks. Shaded  area represents the standard error. In navigation tasks, our method is competitive with other baselines.}
    \label{fig:norm_point_task_performance}
\end{figure*}

We train all the algorithms for 1 million training steps for both velocity-constrained and navigation tasks. We report the normalized final performance in Table \ref{table:performance_velocity} and \ref{table:performance_navigation}. Our proposed algorithm learns a safer policy and outperforms the best baseline in terms of Normalized Cost by a factor of $5.54 \times$ in Walker2d-Velocity, $5.97 \times$  in Swimmer-Velocity, $29.88 \times$ in Ant-Velocity, and $1.86\times$ in Point-Goal1. However, it underperforms in Point-Circle2 and Point-Button1 by $0.12 \times$ and $0.51 \times$ the best baseline algorithm, respectively.  Overall our algorithm achieves a median performance that is $3.7\times$ better than the best baseline algorithm across all environments.

\begin{table}[!htb]
\small
\centering
\setlength{\tabcolsep}{.5mm}
\def\arraystretch{2.5}
\begin{tabular}{@{}|c|ccc|ccc|ccc|@{}}
\toprule
\multirow{2}{*}{Algorithms} & \multicolumn{3}{c|}{Walker2d}                                                                                   & \multicolumn{3}{c|}{Swimmer}                                                                                      & \multicolumn{3}{c|}{Ant}                                                                                        \\ \cmidrule(l){2-10} 
                            & \multicolumn{1}{c|}{Return}          & \multicolumn{1}{c|}{Cost}                     & CVaR Cost                & \multicolumn{1}{c|}{Return}          & \multicolumn{1}{c|}{Cost}                      & CVaR Cost                 & \multicolumn{1}{c|}{Return}          & \multicolumn{1}{c|}{Cost}                     & CVaR Cost                \\ \midrule
BC-Unlabeled                & \multicolumn{1}{c|}{$0.97 \pm 0.01$} & \multicolumn{1}{c|}{$14.72 \pm 4.35$}         & $61.06 \pm 22.71$        & \multicolumn{1}{c|}{$0.87 \pm 0.05$} & \multicolumn{1}{c|}{$89.79 \pm 27.12$}         & $221.96 \pm 84.70$        & \multicolumn{1}{c|}{$1.00 \pm 0.01$} & \multicolumn{1}{c|}{$38.51 \pm 6.29$}         & $128.14 \pm 18.11$       \\
DWBC-NU                     & \multicolumn{1}{c|}{$0.97 \pm 0.00$} & \multicolumn{1}{c|}{$18.91 \pm 9.94$}         & $66.77 \pm 32.52$        & \multicolumn{1}{c|}{$0.82 \pm 0.18$} & \multicolumn{1}{c|}{$118.39 \pm 31.64$}        & $268.70 \pm 82.54$        & \multicolumn{1}{c|}{$1.01 \pm 0.00$} & \multicolumn{1}{c|}{$44.61 \pm 5.99$}         & $126.94 \pm 13.81$       \\
T-REX-WBC                   & \multicolumn{1}{c|}{$0.96 \pm 0.02$} & \multicolumn{1}{c|}{$9.28 \pm 6.57$}          & $34.19 \pm 20.15$        & \multicolumn{1}{c|}{$0.80 \pm 0.01$} & \multicolumn{1}{c|}{$60.43 \pm 8.95$}          & $180.73 \pm 37.45$        & \multicolumn{1}{c|}{$0.99 \pm 0.01$} & \multicolumn{1}{c|}{$25.68 \pm 9.53$}         & $99.02 \pm 29.67$        \\
SafeDICE                    & \multicolumn{1}{c|}{$0.97 \pm 0.00$} & \multicolumn{1}{c|}{$6.91 \pm 4.85$}          & $28.86 \pm 17.17$        & \multicolumn{1}{c|}{$0.88 \pm 0.07$} & \multicolumn{1}{c|}{$142.58 \pm 38.67$}        & $320.88 \pm 80.24$        & \multicolumn{1}{c|}{$1.01 \pm 0.01$} & \multicolumn{1}{c|}{$40.96 \pm 8.35$}         & $119.60 \pm 16.17$       \\
SafeMIL (ours)              & \multicolumn{1}{c|}{$0.99 \pm 0.00$} & \multicolumn{1}{c|}{$\mathbf{1.19 \pm 1.43}$} & $\mathbf{6.95 \pm 7.02}$ & \multicolumn{1}{c|}{$0.87 \pm 0.03$} & \multicolumn{1}{c|}{$\mathbf{10.16 \pm 3.05}$} & $\mathbf{56.48 \pm 9.36}$ & \multicolumn{1}{c|}{$0.95 \pm 0.00$} & \multicolumn{1}{c|}{$\mathbf{0.86 \pm 0.71}$} & $\mathbf{5.32 \pm 2.61}$ \\ \bottomrule
\end{tabular}
\caption{\textbf{Performance Comparison.} Experimental results of velocity-constrained tasks after 1 million training steps. We report Normalized Return (Return), Normalized Cost (Cost), and Normalized CVaR Cost (CVaR Cost). Mean and 95\% CIs over 5 seeds. In velocity-constrained tasks, our method outperforms all the baselines and can recover safer policies without compromising reward performance.}
\label{table:performance_velocity}
\end{table}

\begin{table}[!htb]
\small
\centering
\setlength{\tabcolsep}{.5mm}
\def\arraystretch{2.5}
\begin{tabular}{@{}|c|ccc|ccc|ccc|@{}}
\toprule
\multirow{2}{*}{Algorithms} & \multicolumn{3}{c|}{Circle2}                                                                                     & \multicolumn{3}{c|}{Goal1}                                                                                       & \multicolumn{3}{c|}{Button1}                                                                                        \\ \cmidrule(l){2-10} 
                            & \multicolumn{1}{c|}{Return}          & \multicolumn{1}{c|}{Cost}                     & CVaR Cost                 & \multicolumn{1}{c|}{Return}          & \multicolumn{1}{c|}{Cost}                     & CVaR Cost                 & \multicolumn{1}{c|}{Return}          & \multicolumn{1}{c|}{Cost}                       & CVaR Cost                  \\ \midrule
BC-Unlabeled                & \multicolumn{1}{c|}{$1.08 \pm 0.03$} & \multicolumn{1}{c|}{$83.38 \pm 8.84$}         & $226.96 \pm 10.33$        & \multicolumn{1}{c|}{$1.70 \pm 0.10$} & \multicolumn{1}{c|}{$10.33 \pm 1.76$}         & $49.02 \pm 5.52$          & \multicolumn{1}{c|}{$1.02 \pm 0.09$} & \multicolumn{1}{c|}{$-10.69 \pm 10.81$}         & $85.45 \pm 34.84$          \\
DWBC-NU                     & \multicolumn{1}{c|}{$1.08 \pm 0.02$} & \multicolumn{1}{c|}{$84.69 \pm 7.83$}         & $228.17 \pm 18.70$        & \multicolumn{1}{c|}{$1.64 \pm 0.08$} & \multicolumn{1}{c|}{$15.08 \pm 2.55$}         & $58.75 \pm 7.06$          & \multicolumn{1}{c|}{$1.14 \pm 0.05$} & \multicolumn{1}{c|}{$\mathbf{-15.41 \pm 6.41}$} & $69.71 \pm 18.99$          \\
T-REX-WBC                   & \multicolumn{1}{c|}{$0.88 \pm 0.03$} & \multicolumn{1}{c|}{$\mathbf{4.83 \pm 1.24}$} & $\mathbf{28.05 \pm 5.56}$ & \multicolumn{1}{c|}{$1.59 \pm 0.09$} & \multicolumn{1}{c|}{$10.97 \pm 3.64$}         & $53.96 \pm 7.42$          & \multicolumn{1}{c|}{$0.90 \pm 0.09$} & \multicolumn{1}{c|}{$-15.14 \pm 1.52$}          & $\mathbf{65.10 \pm 12.31}$ \\
SafeDICE                    & \multicolumn{1}{c|}{$1.04 \pm 0.01$} & \multicolumn{1}{c|}{$65.36 \pm 7.72$}         & $195.49 \pm 29.28$        & \multicolumn{1}{c|}{$1.75 \pm 0.11$} & \multicolumn{1}{c|}{$11.08 \pm 1.96$}         & $51.00 \pm 6.84$          & \multicolumn{1}{c|}{$1.08 \pm 0.07$} & \multicolumn{1}{c|}{$-11.07 \pm 5.01$}          & $95.13 \pm 19.57$          \\
SafeMIL (ours)              & \multicolumn{1}{c|}{$0.94 \pm 0.02$} & \multicolumn{1}{c|}{$39.18 \pm 6.42$}         & $106.54 \pm 18.13$        & \multicolumn{1}{c|}{$1.46 \pm 0.1$}  & \multicolumn{1}{c|}{$\mathbf{5.61 \pm 1.75}$} & $\mathbf{40.72 \pm 7.32}$ & \multicolumn{1}{c|}{$0.95 \pm 0.09$} & \multicolumn{1}{c|}{$-10.26 \pm 2.50$}          & $97.16 \pm 20.58$ \\ \bottomrule
\end{tabular}
\caption{\textbf{Performance Comparison.} Experimental results of navigation tasks after 1 million training steps. We report Normalized Return (Return), Normalized Cost (Cost), and Normalized CVaR Cost (CVaR Cost). Mean and 95\% CIs over 5 seeds. In navigation tasks, our method outperforms all baselines in the Point-Goal1 environment and is competitive with other baselines in Point-Circle2 and Point-Button1 environments.}
\label{table:performance_navigation}
\end{table}

\iffalse{
    % We also report the unnormalized plots illustrating the performance comparison between the baselines.
    
    % \begin{figure*}[!htb]
    %     \centering
    %     \includegraphics[width=0.9\textwidth]{images/svg_images/performance_comparison/raw_velocity_task_cmyk.pdf}
    %     \caption{\textbf{Performance Comparison.} Unnormalized experimental results on Walker2d-Velocity, Swimmer-Velocity, Ant-Velocity task. Shaded  area represents the standard error. In velocity-constrained tasks, our method is able to recover safer policies without compromising reward performance.}
    %     \label{fig:raw_velocity_task_performance}
    % \end{figure*}
    
    % \begin{figure*}[!htb]
    %     \centering
    %     \includegraphics[width=0.9\textwidth]{images/svg_images/performance_comparison/raw_point_task_cmyk.pdf}
    %     \caption{\textbf{Performance Comparison.} Unnormalized experimental results on Point-Circle2, Point-Goal1, Point-Button1 tasks. Shaded  area represents the standard error. In navigation tasks, our method is competitive with other baselines.}
    %     \label{fig:raw_point_task_performance}
    % \end{figure*}
}\fi

\subsection{Sensitivity Analysis}
We report the normalized sensitivity to bag size and trajectory length for Point-Goal1 task.
\begin{figure*}[!htb]
    \centering
    \includegraphics[width=0.85\textwidth]{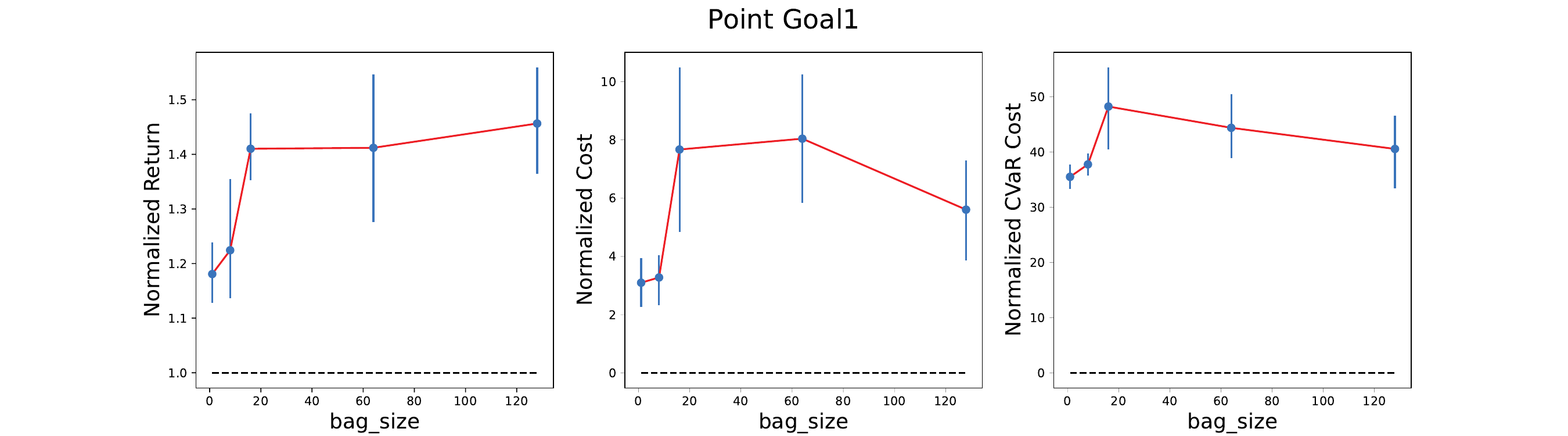}
    \caption{\textbf{Sensitivity to Bag Size.} We report the final mean performance of the algorithm on the Point-Goal1 environment for different bag sizes $K = {1, 8, 16, 64, 128}$, after training for 1 million steps. Mean and 95\% CIs over 5 seeds. We observe that increasing the bag size (K) lead to marginally better safety performance while maintaining reasonable episode return. However, for smaller bag size of 1 and 8, we observe that the safety performance is marginally better. }
    \label{fig:goal_sensitivity_bag_size}
\end{figure*}

\begin{figure*}[!htb]
    \centering
    \includegraphics[width=0.85\textwidth]{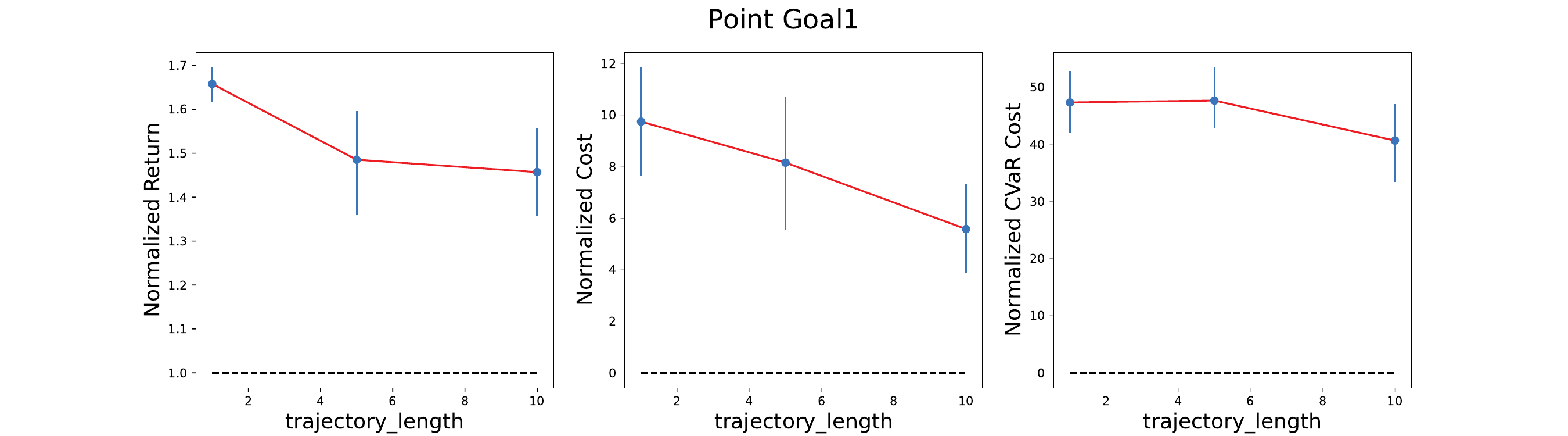}
    \caption{\textbf{Sensitivity to Trajectory Length.} We report the final mean performance of the algorithm on the Point-Goal1 environment for different partial trajectory lengths $H = {1, 5, 10}$, after training for 1 million steps. Mean and 95\% CIs over 5 seeds. Similar to Fig: \ref{fig:swimmer_sensitivity_trajectory_length}, we observe that, for sufficient bag size of 128, safety performance is stable across different trajectory lengths, with slightly better safety performance for longer trajectories.}
    \label{fig:goal_sensitivity_trajectory_length}
\end{figure*}

\newpage
\subsection{Effect of different weighing scheme}
Both transition and trajectory weighing scheme exhibit comparable performance. We report the results on Swimmer-Velocity and Point-Goal1 tasks.

\begin{figure*}[!htb]
    \centering
    \includegraphics[width=0.85\textwidth]{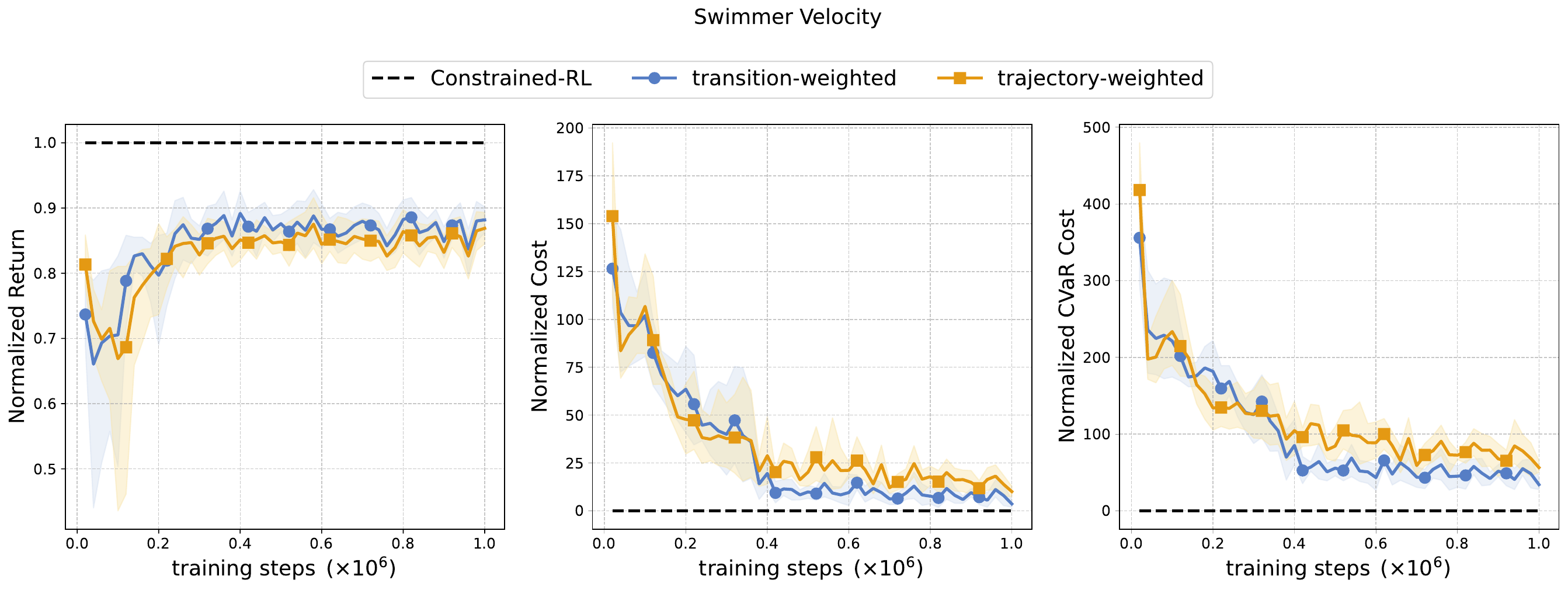}
    \caption{\textbf{Effect of different weighing scheme.} Performance of SafeMIL algorithm on Swimmer-Velocity task with different weighing scheme. The shaded area represents the standard error. We observe that the transition weighing scheme performs marginally better in the Swimmer-Velocity task.}
    \label{fig:ablation_study}
\end{figure*}

\begin{figure*}[!htb]
    \centering
    \includegraphics[width=0.85\textwidth]{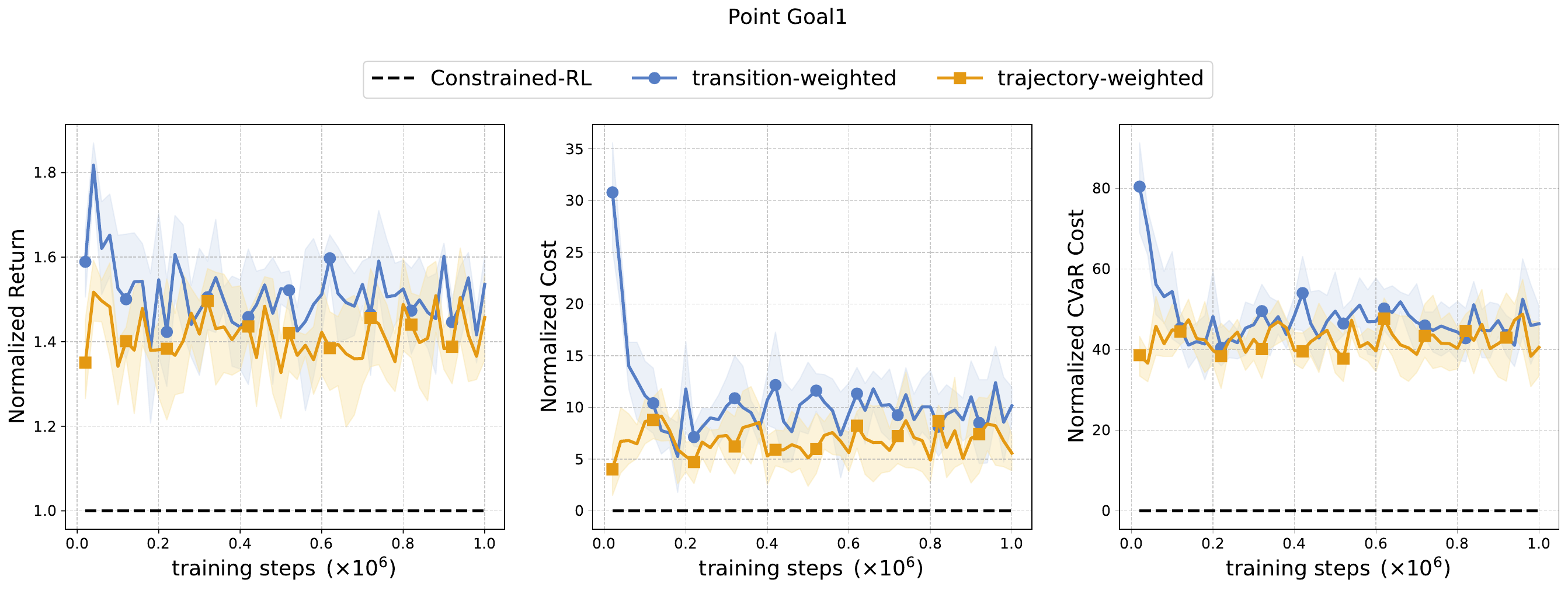}
    \caption{\textbf{Effect of different weighing scheme.} Performance of SafeMIL algorithm on Point-Goal1 task with different weighing scheme. The shaded area represents the standard error. We observe that the trajectory weighing scheme performs marginally better in the Point-Goal1 task.}
    \label{fig:ablation_study_goal1}
\end{figure*}

\subsection{Results of Normalized CVaR $k\%$ Cost}
We report the normalized CVaR $k\%$ Cost for both velocity-constrained and navigation tasks.
\begin{figure*}[!htb]
    \centering
    \includegraphics[width=0.85\textwidth]{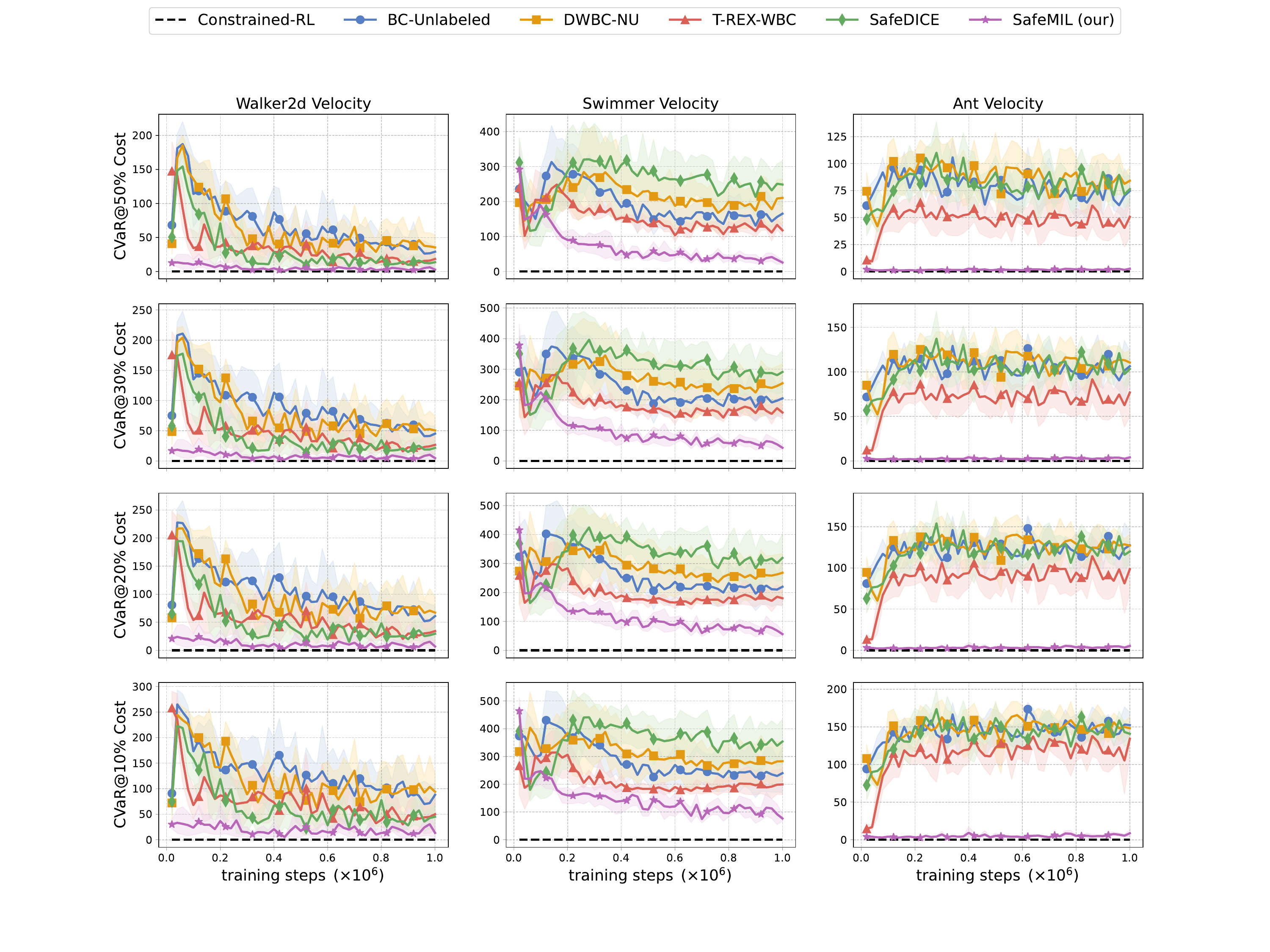}
    \caption{\textbf{Normalized CVaR $k\%$ Cost.} Experimental results on Walker2d-Velocity, Swimmer-Velocity, Ant-Velocity task. Shaded  area represents the standard error. We report the Normalized CVaR @ 50\%, 30\%, 20\%, and 10\% cost. We observe that our findings are consistent and SafeMIL is able to consistently perform better than other baselines.}
    \label{fig:worst_performance_velocity_task}
\end{figure*}

\begin{figure*}[!htb]
    \centering
    \includegraphics[width=0.85\textwidth]{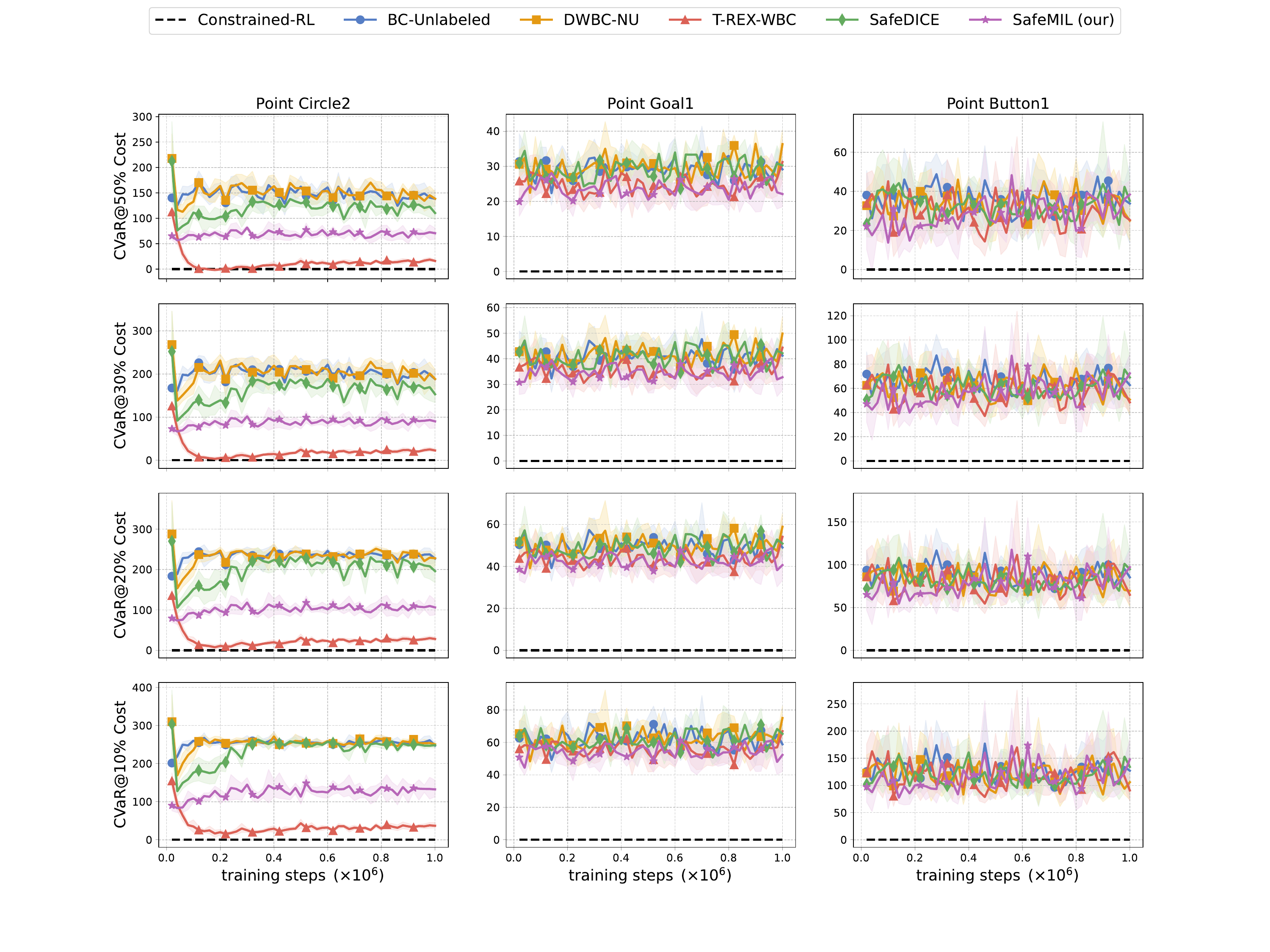}
    \caption{\textbf{Normalized CVaR $k\%$ Cost.} Experimental results on Point-Circle2, Point-Goal1, Point-Button1 task. Shaded  area represents the standard error. We report the Normalized CVaR @ 50\%, 30\%, 20\%, and 10\% cost. We observe that our findings are consistent and SafeMIL is competitive with other baselines.}
    \label{fig:worst_performance_point_task}
\end{figure*}

\section{G.\; Compute Details}
\label{appendix:compute}

We used two server nodes equipped with the following specification:
\begin{itemize}
    \item CPU: AMD EPYC 7543 32-Core Processor
    \item Memory: 30GB
    \item GPU: NVIDIA A30
\end{itemize}

\end{document}